\documentclass{article}
\usepackage{nips07submit_e,times}
\usepackage{amsmath}
\usepackage{amsfonts}
\usepackage[numbers,sort&compress]{natbib}
\usepackage{algorithmic}
\usepackage{algorithm}
\usepackage{graphicx}
\usepackage{subfig}

\usepackage{mdwlist}
\usepackage{tabularx}
\usepackage{multirow}
\usepackage{url}
\usepackage{paralist}

\title{Lower Bounds for BMRM and Faster Rates for Training SVMs}

\input{Definitions}

\author{
Ankan Saha \\
Department of Computer Science\\
University of Chicago\\
Chicago, 60637 \\
\texttt{ankans@cs.uchicago.edu} \\
\And
Xinhua Zhang \\
Computer Sciences Lab\\
Australian National University \\
NICTA, Canberra, Australia \\
\texttt{xinhua.zhang@nicta.com.au} \\
\AND
S V N Vishwanathan \\
Department of Statistics and Computer Science \\
Purdue University \\
West Lafayette, Indiana \\
\texttt{vishy@stat.purdue.edu}
}

\begin{document}

\maketitle
\vspace{-3mm}
\begin{abstract}
  Regularized risk minimization with the binary hinge loss and its
  variants lies at the heart of many machine learning problems. Bundle
  methods for regularized risk minimization (BMRM) and the closely
  related SVMStruct are considered the best general purpose solvers to
  tackle this problem. It was recently shown that BMRM requires
  $O(1/\ve)$ iterations to converge to an $\ve$ accurate
  solution. In the first part of the paper we use the Hadamard matrix to
  construct a regularized risk minimization problem and show that these
  rates cannot be improved. We then show how one can exploit the
  structure of the objective function to devise an algorithm for the
  binary hinge loss which converges to an $\ve$ accurate solution
  in $O(1/\sqrt{\ve})$ iterations.
\end{abstract}

\vspace{-3mm}
\section{Introduction}
\label{sec:Introduction}

Let $\xb_i \in \Xcal \subseteq \RR^d$ denote samples and $\yb_i \in
\Ycal$ be the corresponding labels. Given a training set of $n$ sample
label pairs $\cbr{(\xb_i, \yb_i)}_{i=1}^n$, drawn
i.i.d. from a joint probability distribution on $\Xcal \times \Ycal$,
many machine learning algorithms solve the following regularized risk
minimization problem:
\vspace{-2mm}
\begin{align}
  \label{eq:reg_risk}
  \min_{\wb} J(\wb) := \lambda \Omega(\wb) + R_{\emp}(\wb), \text{ where
  } R_{\emp}(\wb) := \frac{1}{n} \sum_{i=1}^n l(\xb_i, \yb_i; \wb).
  \vspace{-1em}
\end{align}

\vspace{-1em}
Here $l(\xb_i, \yb_i; \wb)$ denotes the loss on instance $(\xb_i,
\yb_i)$ using the current model $\wb$ and $R_{\emp}(\wb)$, the empirical
risk, is the average loss on the training set. The regularizer
$\Omega(\wb)$ acts as a penalty on the complexity of the classifier and
prevents overfitting.  Usually the loss is convex in $\wb$ but can be
nonsmooth while the regularizer is usually a smooth strongly convex
function. Binary Support Vector Machines (SVMs) are a prototypical
example of such regularized risk minimization problems where $\Ycal =
\{1 , -1\}$ and the loss considered is the binary hinge loss:
\begin{align}
  \label{eq:bin-hinge}
  l(\xb_i, y_i; \wb) = \sbr{1-y_i\inner {\wb} {\xb_i}}_{+}, \text{ with }
  \sbr{\cdot}_{+} := \max(0, \cdot).
\end{align}

\vspace{-2mm}
Recently, a number of solvers have been proposed for the regularized
risk minimization problem. The first and perhaps the best known solver
is SVMStruct \cite{TsoJoaHofAlt05}, which was shown to converge in
$O(1/\ve^{2})$ iterations to an $\ve$ accurate solution. The
convergence analysis of SVMStruct was improved to $O(1/\ve)$
iterations by \cite{SmoVisLe07}. In fact, \cite{SmoVisLe07} showed that
their convergence analysis holds for a more general solver than
SVMStruct namely BMRM (Bundle method for regularized risk minimization).

At every iteration BMRM replaces $R_{\emp}$ by a piecewise linear lower
bound $R^{\cp}_{t}$ and optimizes
\begin{align}
  \label{eq:pw-linear}
  \min_{\wb} J_{t}(\wb) := \lambda \Omega(\wb) + R^{\cp}_{t}(\wb),
  \text{where } R^{\cp}_{t}(\wb) := \max_{1 \leq i \leq t}
  \inner{\wb}{\ab_{i}} + b_{i},
\end{align}

\vspace{-1em}
to obtain the next iterate $\wb_{t}$. Here $\ab_{i} \in \partial
R_{\emp}(\wb_{i-1})$ denotes an arbitrary subgradient of $R_{\emp}$ at
$\wb_{i-1}$ (see Section \ref{sec:Preliminaries}) and $b_{i} =
R_{\emp}(\wb_{i-1}) - \inner{\wb_{i-1}}{\ab_{i}}$. The piecewise linear
lower bound is successively tightened until the gap
\vspace{-0.4em}
\begin{align}
  \label{eq:approximation_gap}
  \ve_{t} := \min_{0 \leq t' \leq t} J(\wb_{t'}) - J_t(\wb_t),
\end{align}

\vspace{-1.5em}
falls below a predefined tolerance $\ve$.

Even though BMRM solves an expensive optimization problem at every
iteration, the convergence analysis only uses a simple one-dimensional
line search to bound the decrease in $\ve_{t}$. Furthermore, the
empirical convergence behavior of BMRM is much better than the
theoretically predicted rates on a number of real life problems.  It was
therefore conjectured that the rates of convergence of BMRM could be
improved. In this paper we answer this question in the negative by
explicitly constructing a regularized risk minimization problem for
which BMRM takes at least $O(1/\ve)$ iterations.

One possible way to circumvent the $O(1/\ve)$ lower bound is to solve
the problem in the dual. Using a very old result of Nesterov
\cite{Nesterov83} we obtain an algorithm for SVMs which only requires
$O(1/\sqrt{\ve})$ iterations to converge to an $\ve$ accurate solution;
each iteration of the algorithm requires $O(nd)$ work.  Although we
primarily focus on the regularized risk minimization with the binary
hinge loss, our algorithm can also be used whenever the empirical risk
is piecewise linear and contains a small number of pieces. Examples of
this include multiclass, multi-label, and ordinal regression hinge loss
and other related losses.

\section{Preliminaries}
\label{sec:Preliminaries}

In this paper, lower bold case letters (\eg, $\wb$, $\mub$) denote
vectors, $w_{i}$ denotes the $i$-th component of $\wb$ and $\Delta_k$
refers to the $k$ dimensional simplex. Unless specified otherwise,
$\nbr{\cdot}$ refers to the Euclidean norm $\|\wb\| := \left(
  \sum_{i=1}^{n} w_{i}^{2}\right)^{\frac{1}{2}}$.  $\overline{\RR} :=
\RR \cup \{\infty\}$, and $[t]:= \{1, \ldots, t\}$.  The $\dom$ of a
convex function $F$ is defined by $\dom\ F := \cbr{\wb: F(\wb) <
  \infty}$.  The following three notions will be used extensively:
\begin{definition}
  \label{def:strong-convex}
  A convex function $F:\RR^{n} \to \overline{\RR}$ is strongly convex (s.c.)
  wrt norm $\|\cdot\|$ if there exists a constant $\sigma > 0$ such that $F -
  \frac{\sigma}{2} \|\cdot\|^{2}$ is convex.  $\sigma$ is called the modulus
  of strong convexity of $F$, and for brevity we will call $F$
  $\sigma$-strongly convex or $\sigma$-s.c..
\end{definition}

\begin{definition}
  \label{def:lip-cont-grad}
  A function $F$ is said to have Lipschitz continuous gradient (\lcg) if there
  exists a constant $L$ such that
  \begin{align}
    \label{eq:lip-cont-grad}
    \| \nabla F(\wb) - \nabla F(\wb')\| \leq L \| \wb - \wb'\| \qquad
    \forall\ \wb \text{ and } \wb'.
  \end{align}
For brevity, we will call $F$ $L$-\lcg..
\end{definition}

\begin{definition}
  \label{def:fenchel_dual}
  The Fenchel dual of a function $F:E_1 \to E_2$, is a function
  $F^{\star}:E_2^{\star} \to E_1^{\star}$ given by
  \begin{align}
    \label{eq:fenchel-dual}
    F^{\star}(\wb^{\star}) = \sup_{\wb \in E_1}
    \cbr{\inner{\wb}{\wb^{\star}} - F(\wb)}
  \end{align}
\end{definition}
The following theorem specifies the relationship between strong
convexity of a primal function and Lipschitz continuity of the gradient
of its Fenchel dual.
\begin{theorem}[{\cite[][Theorem 4.2.1 and 4.2.2]{HirLem93}}]
$\phantom{.}$
\label{theorem:SC_LCG}
\begin{enumerate}
\item If $F: \RR^n \to \overline{\RR}$ is $\sigma$-strongly convex, then
  $\dom\ F^{\star} = \RR^n$ and $\grad F^{\star}$ is $\frac{1}{\sigma}$-\lcg.
\item If $F: \RR^n \to \RR$ is convex and $L$-\lcg, then $F^{\star}$ is
  $\frac{1}{L}$-strongly convex.
\end{enumerate}
\end{theorem}

Subgradients generalize the concept of gradients to nonsmooth functions. For
$\wb \in \dom\ F$, $\mub$ is called a subgradient of $F$ at $\wb$ if
\begin{align}
  \label{eq:subgrad-def}
  F(\wb') \geq F(\wb) + \inner{\wb' - \wb}{\mub} \quad \forall \wb'.
\end{align}
The set of all subgradients at $\wb$ is called the subdifferential, denoted by
$\partial F(\wb)$. If $F$ is convex, then $\partial F(\wb) \neq \emptyset$ for
all $\wb \in \dom\ F$, and is a singleton if, and only if, $F$ is differentiable
\citep{HirLem93}.

Any piecewise linear convex function $F(\wb)$ with $t$ linear pieces can
be written as
\begin{align}
  \label{eq:max-linear}
  F(\wb) = \max_{i \in [t]}\{\inner{\ab_i}{\wb} + b_i\},
\end{align}
for some $\ab_{i}$ and $b_{i}$. If the empirical risk $R_{\emp}$ is a
piecewise linear function then the convex optimization problem in
\eqref{eq:reg_risk} can be expressed as
\begin{align}
\label{eq:primal_min}
\min_{\wb} J(\wb) := \min_{\wb}\max_{i \in [t]}\{\inner{\ab_i}{\wb} +
b_i \} + \lambda \Omega(\wb).
\end{align}
Let $\Ab = [\ab_{1}\hdots \ab_{t}]$, then the \emph{adjoint} form of
$J(\wb)$ can be written as
\begin{align}
  \label{eq:dual-problem}
  D(\alphab) := - \lambda \Omega^{\star}(-\lambda^{-1}\Ab\alphab) +
  \inner{\alphab}{b} \text{ with } \alphab \in \Delta_{t}
\end{align}
where the primal and the adjoint optimum are related by
\begin{align}
  \label{eq:primal_dual_relation}
  \wb^{*} = \partial\Omega^{\star}(-\lambda^{-1}\Ab\alphab^{*})
\end{align}
In fact, using concepts of strong duality (see \eg Theorem 2 of
\cite{TeoVisSmoLe09}), it can be shown that
\begin{align}
\label{eq:dual_equiv}
\inf_{\wb \in \RR^d} \cbr{\max_{i \in [n]} \inner{\ab_i}{\wb} + b_i +
  \lambda \Omega (\wb)} = \sup_{\alphab \in \Delta_{t}} \cbr{-\lambda
  \Omega^{\star}(-\lambda^{-1} \Ab \alphab) + \inner{\alphab}{\bb}}
\end{align}

\section{Lower Bounds}
\label{sec:lower_bounds}

The following result was shown by \cite{SmoVisLe07}:
\begin{theorem}[Theorem 4 of \cite{SmoVisLe07}]
  \label{th:upper-bound}
  Assume that $J(\wb) \geq 0$ for all $\wb$, and that
  $\nbr{\partial_{\wb} R_{\emp}(\wb)} \leq G$ for all $\wb \in W$, where
  $W$ is some domain of interest containing all $\wb_{t'}$ for $t' \leq
  t$. Also assume that $\Omega^*$ has bounded curvature, i.e.\ let
  $\nbr{\partial^2_\mu \Omega^*(\mub)} \leq H^*$ for all $\mub \in
  \cbr{-\lambda^{-1} A \alphab \text{ where } \alphab \in \Delta_{t}}$.
  Then, for any $\ve < 4 G^2 H^* / \lambda$ we have $\ve_{t} <
  \ve$ after at most
  \begin{align}
    \log_{2} \frac{\lambda J(\zero)}{G^2 H^*} + \frac{8 G^2 H^*}{\lambda
      \ve} - 4
  \end{align}
  steps.
\end{theorem}
Although the above theorem proves an upper bound of $O(1/\ve)$ on the
number of iterations, the tightness of this bound has been an open
question. We now demonstrate a function which satisfies all the
conditions of the above theorem, and yet takes $\Omega(1/\ve)$
iterations to converge.

To construct our lower bounds we make use of the Hadamard matrix. An $n
\times n$ Hadamard matrix is an orthogonal matrix with $\{\pm1\}$
elements which is recursively defined for $d = 2^{k}$ (for some $k$):
\begin{center}
  \begin{tabular}{ccc}
    $\Hb_{1} = \left( +1 \right)$
    &
    $\Hb_{2} = \left(
      \begin{array}{cc}
        +1 & +1 \\
        +1 & -1
      \end{array}
    \right)$
    &
    $\Hb_{2d} = \left(
      \begin{array}{cc}
        \Hb_{d} & \;\;\Hb_{d} \\
        \Hb_{d} & -   \Hb_{d}
      \end{array}
    \right)$
  \end{tabular}
\end{center}
Note that all rows of the Hadamard matrix $\Hb_{d}$ are orthogonal and
have Euclidean norm $\sqrt{d}$. Consider the following $d \times d$
orthonormal matrix
\begin{align*}
  \Ab := \frac{1}{\sqrt{d}}\left(
    \begin{array}{cc}
      \Hb_{d/2} & -\Hb_{d/2} \\
      -\Hb_{d/2} & \Hb_{d/2}
    \end{array} \right),
\end{align*}
whose columns $\ab_{i}$ are orthogonal and have Euclidean norm $1$,
which is used to define the following piecewise quadratic function:
\begin{align}
  \label{eq:lower_approx}
  J(\wb) &= \underbrace{\max_{i \in [d]}\inner{\ab_i}{\wb}}_{R_{\emp}} +
  \underbrace{\frac{\lambda}{2}\| \wb \|^{2}}_{\lambda \Omega(\wb)}.
\end{align}
\begin{theorem}
  \label{th:lower-bound}
  The function $J(\wb)$ defined in \eqref{eq:lower_approx} satisfies all
  the conditions of Theorem \ref{th:upper-bound}. For any $t <
  \frac{d}{2}$ we have $\ve_{t} \geq \frac{1}{2 \lambda t}$.
\end{theorem}
\begin{proof}
  By construction of $\Ab$, we have $\max_{i \in [d]}\inner{\ab_i}{\wb}
  \geq 0$. Furthermore, $\frac{\lambda}{2}\| \wb \|^{2} \geq
  0$. Together this implies that $J(\wb) \geq 0$. Furthermore,
  $\partial_{\wb} R_{\emp}(\wb)$ are the columns of $\Ab$, which implies
  that $\nbr{\partial_{\wb} R_{\emp}(\wb)} \leq 1$. Since we set
  $\Omega(\cdot) = \frac{1}{2} \|\cdot\|^{2}$ it follows that
  $\Omega^{*}(\cdot) = \frac{1}{2} \|\cdot\|^{2}$. Therefore
  $\nbr{\partial^2_\mu \Omega^*(\cdot)} = \frac{1}{2}$. Hence,
  \eqref{eq:lower_approx} satisfies all the conditions of Theorem
  \ref{th:upper-bound}.

  Let $\wb_0, \wb_1\hdots \wb_t$ denote the solution produced by BMRM
  after $t$ iterations, and let $\ab_{i_{0}}, \ab_{i_{1}} \ldots
  \ab_{i_{t}}$ denote the corresponding subgradients. Then
\begin{align}
  \label{eq:pw-lower_approx}
  J_t(\wb) &= \max_{j \in [t]}\inner{\ab_{i_{j}}}{\wb} +
  \frac{\lambda}{2}\| \wb \|^{2} \text { with } \wb_t = \argmin_{\wb \in
    \RR^d}J_t(\wb).
\end{align}
If we define $\Ab_{t} = [\ab_{i_{j}}]$ with $j \in [t]$ then
\begin{align}
  \label{eq:minimax}
  \nonumber J_t(\wb_t) &= \min_{\wb \in \RR^d}\max_{j \in
    [t]}\inner{\ab_{i_{j}}}{\wb} + \frac{\lambda}{2}\|\wb \|^{2} =
  \min_{\wb \in \RR^d}\max_{\alphab \in \Delta^t}\wb^{\top}\Ab_{t}
  \alphab +
  \frac{\lambda}{2}\|\wb \|^{2} \\
  &= \max_{\alphab \in \Delta^t}\min_{\wb \in \RR^d}\wb^{\top}\Ab_{t}
  \alphab + \frac{\lambda}{2}\|\wb \|^{2} \hspace{3mm}
  \text{( See Appendix \ref{sec:MinimaxTheoremConvex}
    for details)} \\
  \nonumber &= \max_{\alphab \in \Delta^t}
  -\frac{1}{2\lambda}\alphab^{\top}\Ab_{t}^{\top}\Ab_{t}\alphab =
  -\min_{\alphab \in \Delta^t}
  \frac{1}{2\lambda}\alphab^{\top}\Ab_{t}^{\top}\Ab_{t}\alphab.
\end{align}
Since the columns of $\Ab$ are orthonormal, it follows that
$\Ab_{t}^{\top} \Ab = \Ib_{t}$ where $\Ib_t$ is the $t$ $\times$ $t$
dimensional identity matrix. Thus
\begin{align*}
  J_t(\wb_t) = -\frac{1}{2\lambda}\min_{\alpha \in \Delta^t}\|\alphab
  \|^{2} = -\frac{1}{2\lambda}\frac{1}{t}
\end{align*}
Combining this with $J(\wb_{t'}) \geq 0$ for $t' \in [t]$ and recalling
the definition of $\ve_{t}$ from \eqref{eq:approximation_gap}
completes the proof.
\end{proof}
In fact, $\ve_{t}$ is a proxy for the primal gap
\begin{align*}
  \delta_{t} = \min_{t' \in [t]} J(\wb_{t'}) - J(\wb^{*}).
\end{align*}
Since $J(\wb^{*})$ is unknown, it is replaced by $J_{t}(\wb_{t})$ to
obtain $\ve_{t}$. Since $J(\wb^{*}) \geq J_{t}(\wb_{t})$, it
follows that $\ve_{t} \geq \delta_{t}$ \citep{TeoVisSmoLe09}. We
now show that Theorem \ref{th:lower-bound} holds even if we replace
$\ve_{t}$ by $\delta_{t}$.
\begin{theorem}
  \label{lem:tight-lb}
  Under the same assumptions as Theorem \ref{th:lower-bound}, for any $t
  < \frac{d}{2}$ we have $\delta_{t} \geq \frac{1}{2 \lambda t}$.
\end{theorem}
\begin{proof}
  Note that $J_{t}(\wb)$ is minimized by setting $\alpha = \frac{1}{t}
  \eb$, where $\eb$ denotes the $t$ dimensional vector of
  ones. Recalling that $\Omega^{*}(\cdot) = \frac{1}{2} \|\cdot\|^{2}$
  and using \eqref{eq:primal_dual_relation} one can write $\wb_{t} =
  -\frac{1}{t \lambda} \Ab_{t} \eb$. Since the columns of $\Ab$ are
  orthonormal, it follows that $\inner{\ab_{i}}{\wb_{t}}$ is
  $-\frac{1}{t \lambda}$ if $\ab_{i} $ is a column of $\Ab_{t}$ and 0
  otherwise. Therefore, $\max_{i \in [d]} \inner{\ab_{i}}{\wb_{t}}
  =0$. On the other hand, by noting that $\Ab_{t}^{\top} \Ab = \Ib_{t}$
  we obtain $\frac{\lambda}{2} ||\wb_{t}||^{2} = \frac{1}{2t
    \lambda}$. Plugging into \eqref{eq:lower_approx} yields $J(\wb_{t})
  = \frac{1}{2t \lambda}$ and hence $\min_{t' \in [t]} J(\wb_{t'}) =
  \frac{1}{2t \lambda}$. It remains to note that $J(\wb) \geq 0$, while
  $J(\zero) = 0$. Therefore $J(\wb^{*}) = J(\zero) = 0$.
\end{proof}

\section{A new algorithm with convergence rates $O(1/\sqrt{\ve})$}
\label{sec:new_algorithm}


We now turn our attention to the regularized risk minimization with the
binary hinge loss, and propose a new algorithm. Our algorithm is based
on \cite{Nesterov83} and \cite{Nesterov03} which proposed a non-trivial
scheme of minimizing an $L$-\lcg\ function to $\ve$-precision in
$O(1/\sqrt{\ve})$ iterations. Our contributions are two fold. First, we
show that the dual of the regularized risk minimization problem is
indeed a $L$-\lcg\ function. Second, we introduce an $O(n)$ time algorithm for
projecting onto an  $n$-dimensional simplex or in general an $n$-dimensional
box with a single linear equality constraint, thus improving upon the
$O(n\log n)$ deterministic algorithm of \cite{DucShaSigCha08} (which also
gives a randomized algorithm having expected complexity $O(n)$).
This projection is repeatedly invoked as a subroutine by Nesterov's
algorithm when specialized to our problem.



Consider the problem of minimizing a function $J(\wb)$ with the
following structure over a closed convex set $Q_{1}$:
\begin{equation}
  \label{eq:primal_nest03}
  J(\wb) = f(\wb) + g^{\star}(\Ab\wb).
\end{equation}
Here $f$ is strongly convex on $Q_1$, $\Ab$ is a linear operator which
maps $Q_1$ to another closed convex set $Q_{2}$, and $g$ is convex and
\lcg\ on $Q_2$. \cite{Nesterov03} works with the adjoint form of $J$:
\begin{equation}
\label{eq:dual_nest03}
D(\alphab) = -g(\alphab) - f^{\star}( -\Ab^{\top} \alphab),
\end{equation}
which is \lcg\ according to Theorem \ref{theorem:SC_LCG}.  Under some
mild constraint qualifications which we omit for the sake of brevity
(see \eg\ Theorem 3.3.5 of \cite{BorLew00}) we have
\begin{align}
  \label{eq:strong-convex}
  J(\wb) \ge D(\alphab) \text{ and } \inf_{\wb \in Q_1} J(\wb) =
  \sup_{\alphab \in Q_2} D(\alphab).
\end{align}
By using the algorithm in \cite{Nesterov83} to maximize $D(\alphab)$ one
can obtain an algorithm which converges to an $\ve$ accurate solution of
$J(\wb)$ in $O(1/\sqrt{\ve})$ iterations.

The regularized risk minimization with the binary hinge loss can be
identified with \eqref{eq:primal_nest03} by setting
\begin{align}
  \label{eq:primal_dual_svm}
  J(\wb) = \underbrace{\frac{\lambda}{2} \nbr{\wb}^2}_{f(\wb)}
  + \underbrace{\min_{b \in \RR} \frac{1}{n} \sum_{i=1}^n \sbr{1 - y_i
      (\inner{\xb_i}{\wb} + b) }_+}_{g^{\star}(\Ab \wb)}
\end{align}
The latter, $g^{\star}$, is the dual of $g(\alphab) = -\sum_{i}\alpha_i$ (see
Appendix \ref{sec:Derivgstar}).  Here $Q_{1} = \RR^{d}$. Let $\Ab := -\Yb
\Xb^{\top}$ where $\Yb := \diag (y_1, \ldots, y_n)$, $\Xb :=
[\xb_1, \ldots, \xb_n]$. Then the adjoint can be written as :
\begin{align}
  \label{eq:adjoint}
  D(\alphab) & := -g(\alphab) - f^{\star}(-\Ab^{\top} \alphab) =  \sum_i
  \alpha_i - \frac{1}{2 \lambda} \alphab^{\top} \Yb \Xb^{\top} \Xb \Yb
  \alphab \qquad \text{ with } \\
  \label{eq:q2}
  Q_2 & = \cbr{\alphab \in [0, n^{-1}]^n : \sum_i y_i \alpha_i = 0}.
\end{align}
In fact, this is the well known SVM dual objective function with the
bias incorporated.

Now we present the algorithm of \cite{Nesterov03} in Algorithm \ref{algo:nesterov03}.  Since it optimizes the primal $J(\wb)$ and the adjoint $D(\alphab)$ simultaneously, we call it \nest\ (PRimal-Adjoint GAp Minimization).  It requires a $\sigma_2$-strongly convex prox-function on $Q_2$: $d_2(\alphab) = \frac{\sigma_2}{2}\|\alphab\|^{2}$,
and sets $D_2 = \max_{\alphab \in Q_2} d_2(\alphab)$.  Let the Lipschitz constant of $\nabla D(\alphab)$
be $L$.
Algorithm \ref{algo:nesterov03} is based on two mappings $\alphab_\mu(\wb): Q_1 \mapsto Q_2$
and $\wb(\alphab): Q_2 \mapsto Q_1$, together with an auxiliary mapping $v: Q_2 \mapsto Q_2$.  They are
defined by
\begin{align}
  \label{eq:alpha_mu}
  \alphab_{\mu} (\wb) &:= \argmin_{\alphab \in Q_2} \mu d_2(\alphab) -
  \inner{\Ab \wb}{\alphab} + g(\alphab) = \argmin_{\alphab \in Q_2} \frac{\mu}{2} \nbr{\alphab}^2
  + \wb^{\top} \Xb \Yb \alphab - \sum_i \alpha_i, \\
  \label{eq:w_alpha}
  \wb(\alphab) &:= \argmin_{\wb \in Q_1} \inner{\Ab \wb}{\alphab} + f(\wb) = \argmin_{\wb \in \RR^d}
  -\wb^{\top} \Xb \Yb \alphab + \frac{\lambda}{2} \nbr{\wb}^2 = \frac{1}{\lambda}
  \Xb \Yb \alphab, \\
  \label{eq:v_alpha}
  v(\alphab) &:= \argmin_{\alphab' \in Q_2} \frac{L}{2} \nbr{\alphab' - \alphab}^2 -
  \inner{\grad D(\alphab)}{\alphab' - \alphab}.
\end{align}
Equations \eqref{eq:alpha_mu} and \eqref{eq:v_alpha} are examples of a box constrained QP with a single
equality constraint. In the appendix, we provide a linear time algorithm to find the minimizer of
such a QP. The overall complexity of each iteration is thus $O(nd)$ due to the gradient calculation
in \eqref{eq:v_alpha} and the matrix multiplication in \eqref{eq:w_alpha}.
\begin{algorithm}[t]
\begin{algorithmic}[1]
    \caption{\label{algo:nesterov03} \nest: an $O(1/k^2)$ rate primal-adjoint solver.}
            \REQUIRE $L$ as a conservative estimate of (\ie, no less than) the Lipschitz constant of $\grad D(\alphab)$.
            \ENSURE Two sequences $\wb_k$ and $\alphab_k$ which reduce the duality gap at $O(1/k^2)$
            rate.
            \STATE Initialize: Randomly pick $\alphab_{-1}$ in $Q_2$.  Let $\mu_0 = 2 L$,  $\alphab_0
            \leftarrow v(\alphab_{-1})$,  $\wb_0 \leftarrow \wb(\alphab_{-1})$.
            \FOR{$k = 0, 1, 2, \ldots$}
                \STATE Let $\tau_k = \frac{2}{k+3}$,  $\betab_k \leftarrow (1 - \tau_k) \alphab_k +
                \tau_k \alphab_{\mu_k} (\wb_k)$.
                \STATE Set $\wb_{k+1} \leftarrow (1 - \tau_k) \wb_k + \tau_k \wb(\betab_k)$,
                  $\alphab_{k+1} \leftarrow v(\betab_k)$,   $\mu_{k+1} \leftarrow (1 - \tau_k) \mu_k$.
           \ENDFOR
\end{algorithmic}
\end{algorithm}

\subsection{Convergence Rates}
\label{subsec:convergence_rate}

According to \cite{Nesterov03}, on running Algorithm \nest\ for $k$ iterations,
the $\alphab_k$ and $\wb_k$ satisfy:
\begin{align}
\label{eq:nes03_rate}
J(\wb_k) - D(\alphab_k) \le \frac{4 L D_2}{(k+1)(k+2) \sigma_2}.
\end{align}
For SVMs,  $L =  \frac{1}{\lambda} \nbr{\Ab}_{1,2}^2$ where $\nbr{\Ab}_{1,2} = \max \cbr{\inner{\Ab \wb}{\alphab} :
\nbr{\alphab} = 1, \nbr{\wb} = 1}$, $\sigma_2 = 1$, $D_2 = \frac{1}{2n}$. Assuming $\nbr{\xb_i} \le R$,
\begin{align*}
  \abr{\inner{\Ab \wb}{\alphab}}^2  &\le \nbr{\alphab}^2 \nbr{\Yb \Xb^{\top} \wb}^2
  = \nbr{\Xb^{\top} \wb}^2
  = \sum_i (\xb_i^{\top} \wb)^2
  \le \sum_i \nbr{\wb}^2 \nbr{\xb_i}^2 \le n R^2.
\end{align*}
Thus by \eqref{eq:nes03_rate}, we conclude
\vspace{-0.5em}
\begin{align*}
  J(\wb_k) - D(\alphab_k) \le \frac{4 L D_2}{(k+1)(k+2) \sigma_2} \le
  \frac{2 R^2}{\lambda(k+1)(k+2)} < \ve, \qquad \text{ which gives } k > \left(\frac{R}{\sqrt{\lambda \ve}}\right).
\end{align*}
\vspace{-0.5em}

It should be noted that our algorithm has a better dependence on
$\lambda$ compared to other state-of-the-art SVM solvers like Pegasos
\cite{ShaSinSre07}, SVM-Perf \cite{Joachims06}, and BMRM
\citep{TeoVisSmoLe09} which have a factor of $\frac{1}{\lambda}$ in
their convergence rates.  Our rate of convergence is also data dependent, showing how the correlation of the dataset
$\Yb \Xb = (y_1 \xb_1, \ldots, y_n \xb_n)$ affects the rate via the Lipschitz constant $L$, which is equal to the
square of the maximum singular value of $\Yb \Xb$ (or the maximum eigenvalue of $\Yb \Xb \Xb^{\top} \Yb$).  On one
extreme, if $\xb_i$ is the $i$-th dimensional unit vector then $L=1$, while $L = n$ if all $y_i \xb_i$ are identical.

\subsection{Structured Data}
It is noteworthy that applying \nest\ to structured data is straightforward.  Due to space constraints, we present the details in Appendix \ref{app:struct_output}.  A key interesting problem there is how to project onto a probability simplex such that the image decomposes according to a graphical model.

\section{Experimental Results}
\label{sec:ExperimentalResults}

In this section, we compare the empirical performance of our \nest\ with state-of-the-art binary linear SVM solvers, including \liblinear\footnote{http://www.csie.ntu.edu.tw/$\sim$cjlin/liblinear} \cite{HsiChaLinKeeetal08}, \pegasos\footnote{http://ttic.uchicago.edu/$\sim$shai/code/pegasos.tgz} \cite{ShaSinSre07}, and \BMRM\footnote{http://users.rsise.anu.edu.au/$\sim$chteo/BMRM.html} \cite{TeoVisSmoLe09}.

\paragraph{Datasets}

Table \ref{tab:dataset} lists the statistics of the dataset.  \verb#adult9#, \verb#astro-ph#,
\verb#news20#, \verb#real-sim#, \verb#reuters-c11#, \verb#reuters-ccat# are from the same
source as in \cite{HsiChaLinKeeetal08}.  \verb#aut-avn# classifies documents on auto and aviation
(http://www.cs.umass.edu/$\sim$mccallum/data/sraa.tar.gz).   \verb#covertype# is from UCI repository.
We did not normalize the feature vectors and no bias was used.

\begin{table}
\footnotesize
\setlength{\tabcolsep}{2pt}
\begin{tabularx}
{\linewidth}%
    {>{\setlength{\hsize}{0.105 \hsize}}X|%
    >{\setlength{\hsize}{0.085\hsize}}X|%
    >{\setlength{\hsize}{0.085\hsize}}X|%
    >{\setlength{\hsize}{0.065\hsize}}X||%
    >{\setlength{\hsize}{0.12\hsize}}X|%
    >{\setlength{\hsize}{0.095\hsize}}X|%
    >{\setlength{\hsize}{0.13\hsize}}X|%
    >{\setlength{\hsize}{0.065\hsize}}X||%
    >{\setlength{\hsize}{0.145\hsize}}X|%
    >{\setlength{\hsize}{0.09\hsize}}X|%
    >{\setlength{\hsize}{0.12\hsize}}X|%
    >{\setlength{\hsize}{0.045\hsize}}X}%
    dataset & $n$ & $d$ & $s$(\%) &  dataset & $n$ & $d$ & $s$(\%) &  dataset & $n$ & $d$ & $s$(\%) \\
    \hline
    adult9 & 32,561 & 123 & 11.28 & covertype & 522,911 & 6,274,932 & 22.22 & reuters-c11 & 23,149 & 1,757,801 & 0.16 \\
    astro-ph & 62,369 & 99,757 & 0.077 & news20 & 15,960 & 7,264,867 & 0.033 & reuters-ccat & 23,149 & 1,757,801 & 0.16 \\
    aut-avn & 56,862 & 20,707 & 0.25 & real-sim & 57,763 & 2,969,737 & 0.25 & web8 & 45,546 & 579,586 & 4.24
\end{tabularx}
\normalsize
\vspace{-0.2em}
\caption{Dataset statistics. $n$: \#examples, $d$: \#features, $s$: feature density.}
\label{tab:dataset}
\vspace{-1em}
\end{table}

\paragraph{Algorithms}

Closest to \nest\ in spirit is the line search BMRM (\lsbmrm) which minimizes
the current piecewise lower bound of regularized $R_{\emp}$ via a one dimensional
line search between the current $\wb_t$ and the last subgradient.  This simple
update was enough for \cite{SmoVisLe07} to prove the $1/\ve$ rate of convergence.
Interpreted in the adjoint form, this update corresponds to coordinate descent
with the coordinate being chosen by the Gauss-Southwell rule \cite{Bollen84}.  In contrast,
\nest\ performs a parallel update of all coordinates in each iteration and achieves
faster convergence rate.  So in this section, our main focus is to show that
\nest\ converges faster than \lsbmrm.

We also present the results of \liblinear\ and \pegasos. \liblinear\ is a dual
coordinate descent optimizer for linear SVMs.  \pegasos\ is a primal estimated
sub-gradient solver for SVM with L1 hinge loss.  We tested two extreme variants of
\pegasos: \pegan\ where all the training examples are used in each iteration, and \pegaone\
where only one randomly chosen example is used.  Finally, we also compare with the
\qpbmrm\ proposed in \cite{TeoVisSmoLe09} which solves the full QP in
\eqref{eq:dual-problem} in each iteration.

It should be noted that SVM$^{struct}$ \cite{TsoJoaHofAlt05} is also a general purpose
regularized risk minimizer, and when specialized to binary SVMs, the SVMPerf
\cite{Joachims05, Joachims06} gave the first linear time algorithm for training linear SVMs.
We did not compare with SVMPerf \cite{Joachims06} because its cutting plane nature is
very similar to BMRM when specialized to binary linear SVMs.

For \nest, since the Lipschitz constant $L$ of the gradient of the SVM dual is
unknown in practice, we resort to \cite{Nesterov07} which automatically estimates $L$ while the
rates presented in Section \ref{subsec:convergence_rate} are unchanged.  We further implemented
\nestb, the \nest\ algorithm which uses SVM bias. In this case the inner optimization is a
QP with box constraints and a single linear equality constraint.

For all datasets, we obtained the best $\lambda \in \cbr{2^{-20}, \ldots, 2^{0}}$  using their
corresponding validation sets, and the chosen $\lambda$'s are given in Appendix \ref{app:exp}.

\paragraph{Results}

\begin{figure}[t]
\begin{centering}
\subfloat[astro-ph]{
    \includegraphics[width=4.35cm]{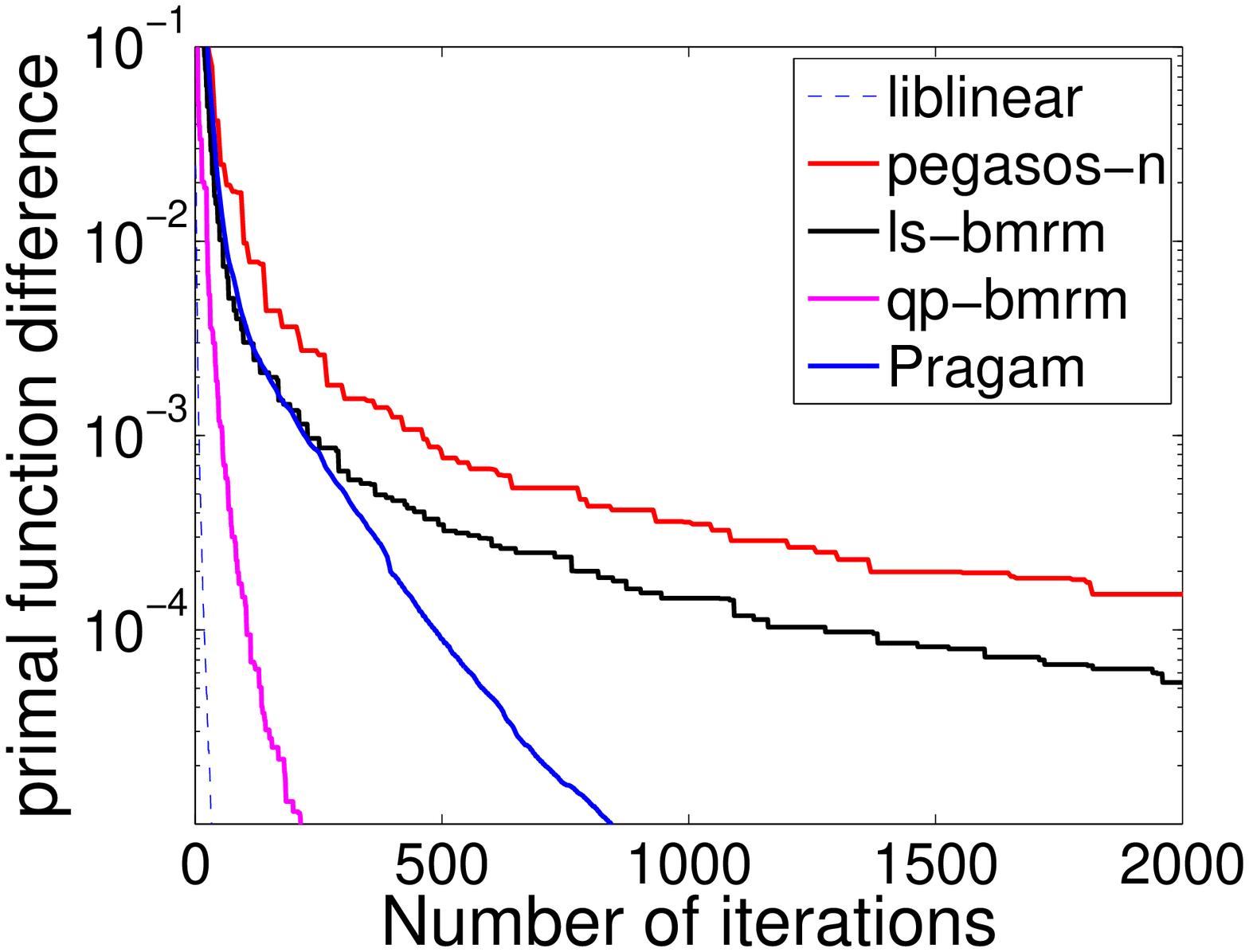}} ~~
\subfloat[news20]{
    \includegraphics[width=4.35cm]{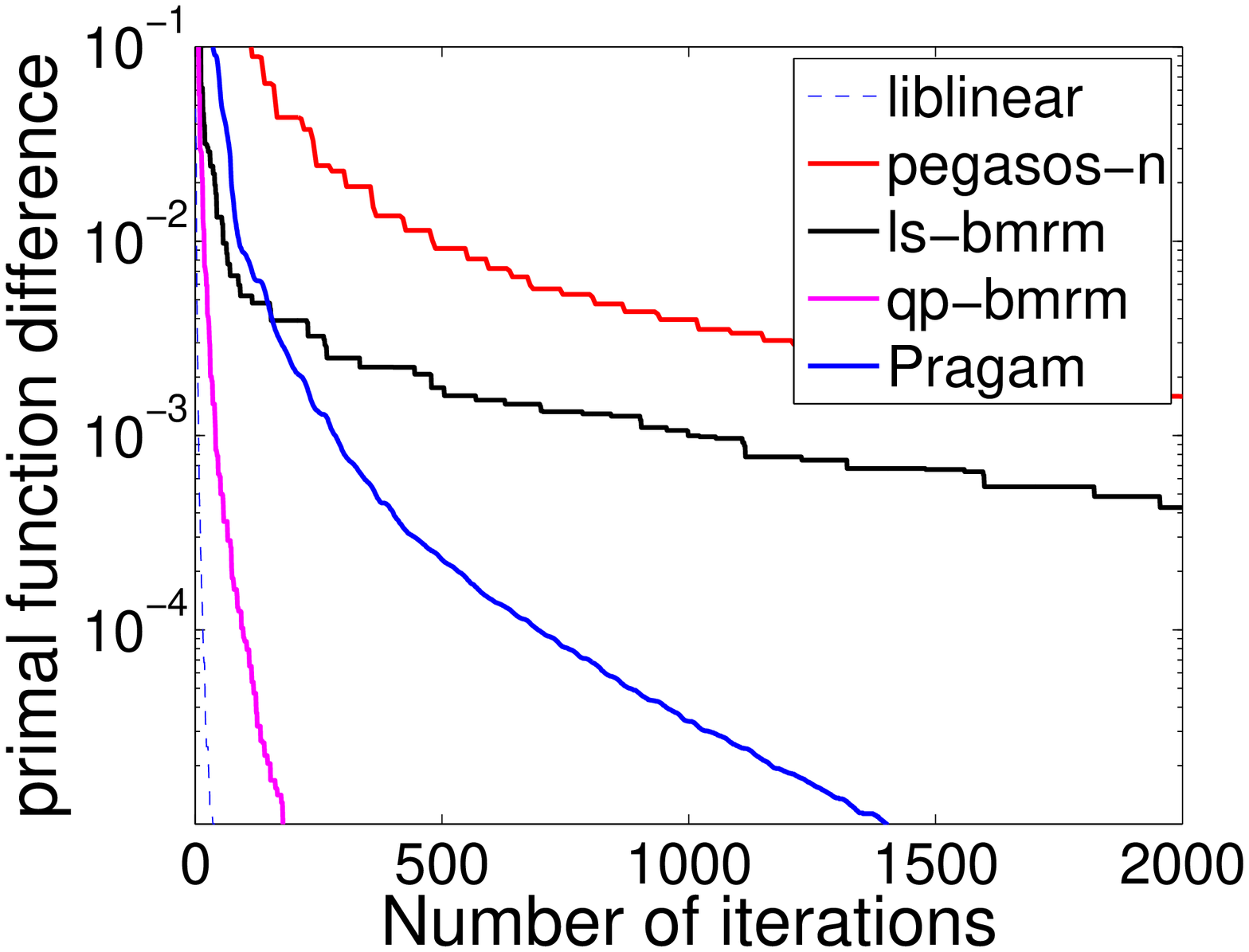}} ~~
\subfloat[real-sim]{
    \includegraphics[width=4.35cm]{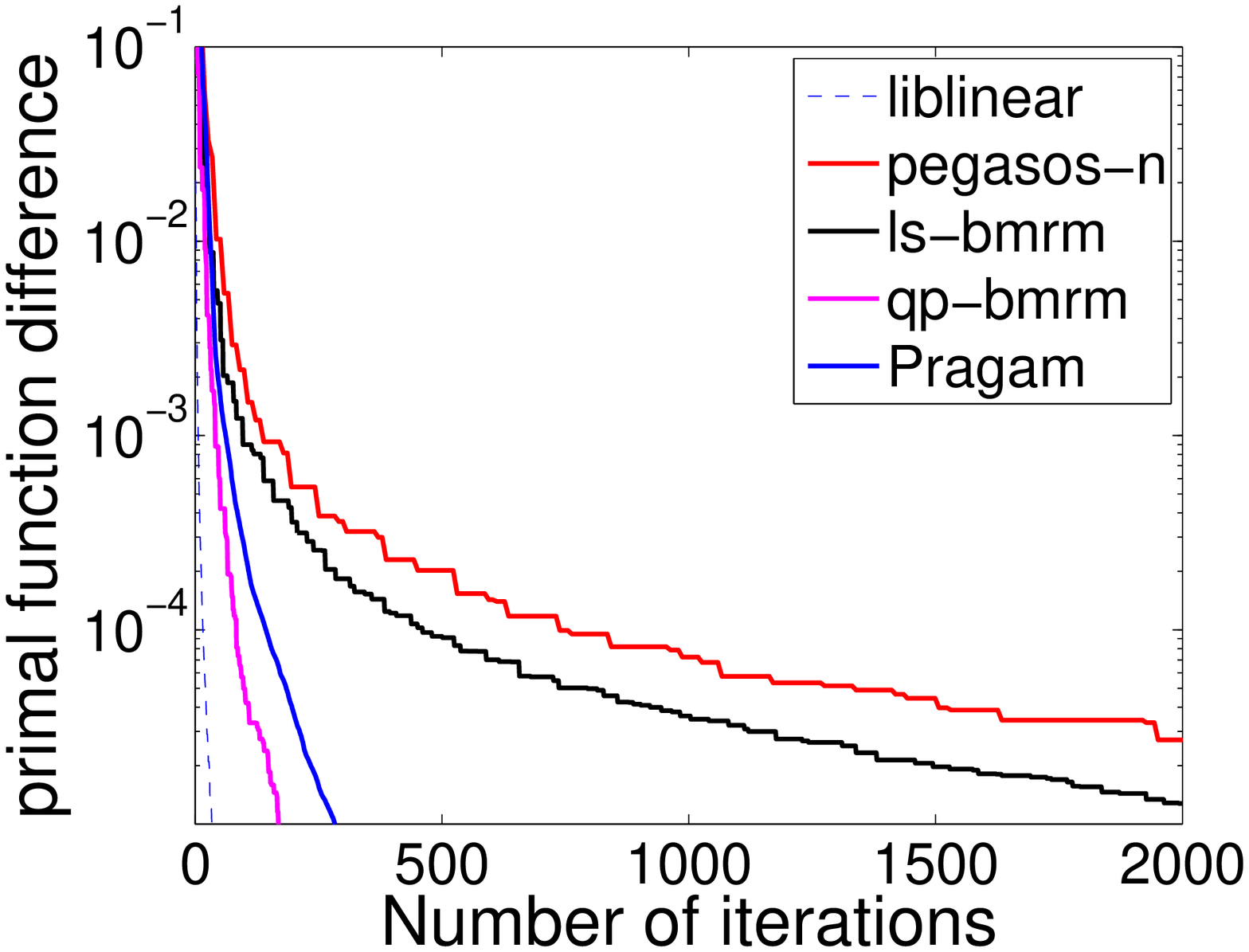}} \\
\vspace{-0.5em}
\caption{Primal function error versus number of iterations.}
\label{fig:fmin_diff_vs_iter_text}
\vspace{-1em}
\end{centering}
\end{figure}

\begin{figure}[t]
\begin{centering}
\subfloat[astro-ph]{
    \includegraphics[width=4.35cm]{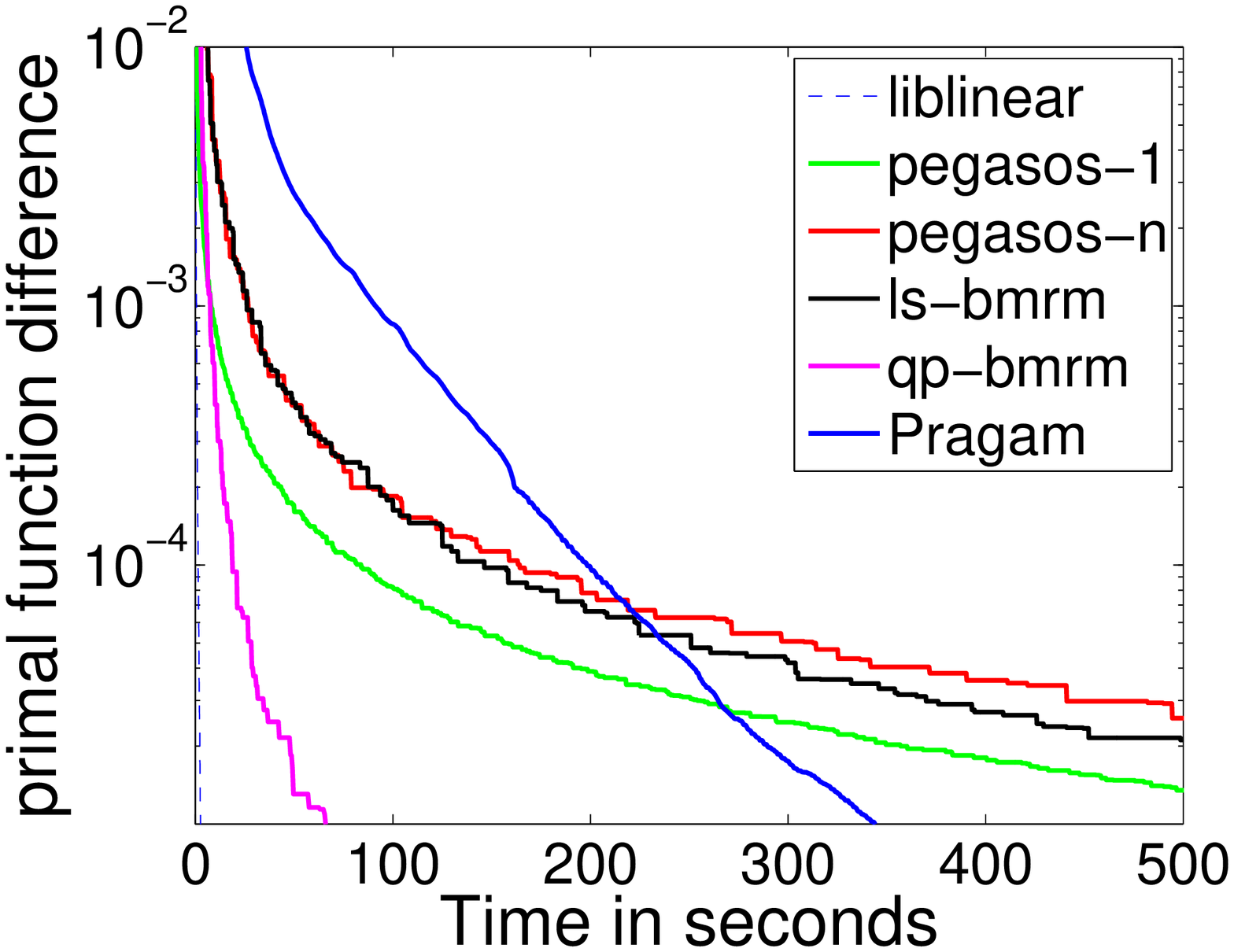}} ~~
\subfloat[news20]{
    \includegraphics[width=4.35cm]{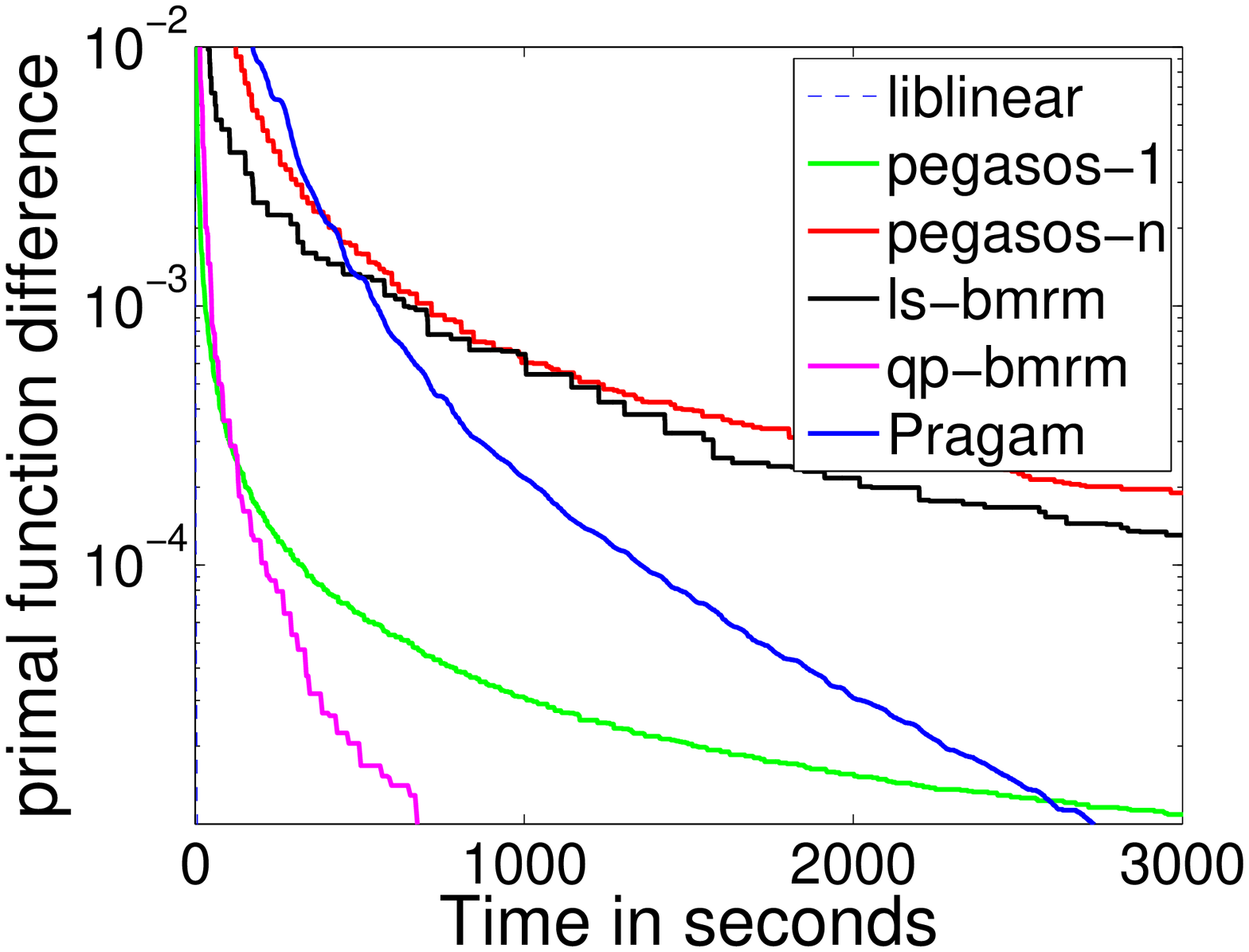}} ~~
\subfloat[real-sim]{
    \includegraphics[width=4.35cm]{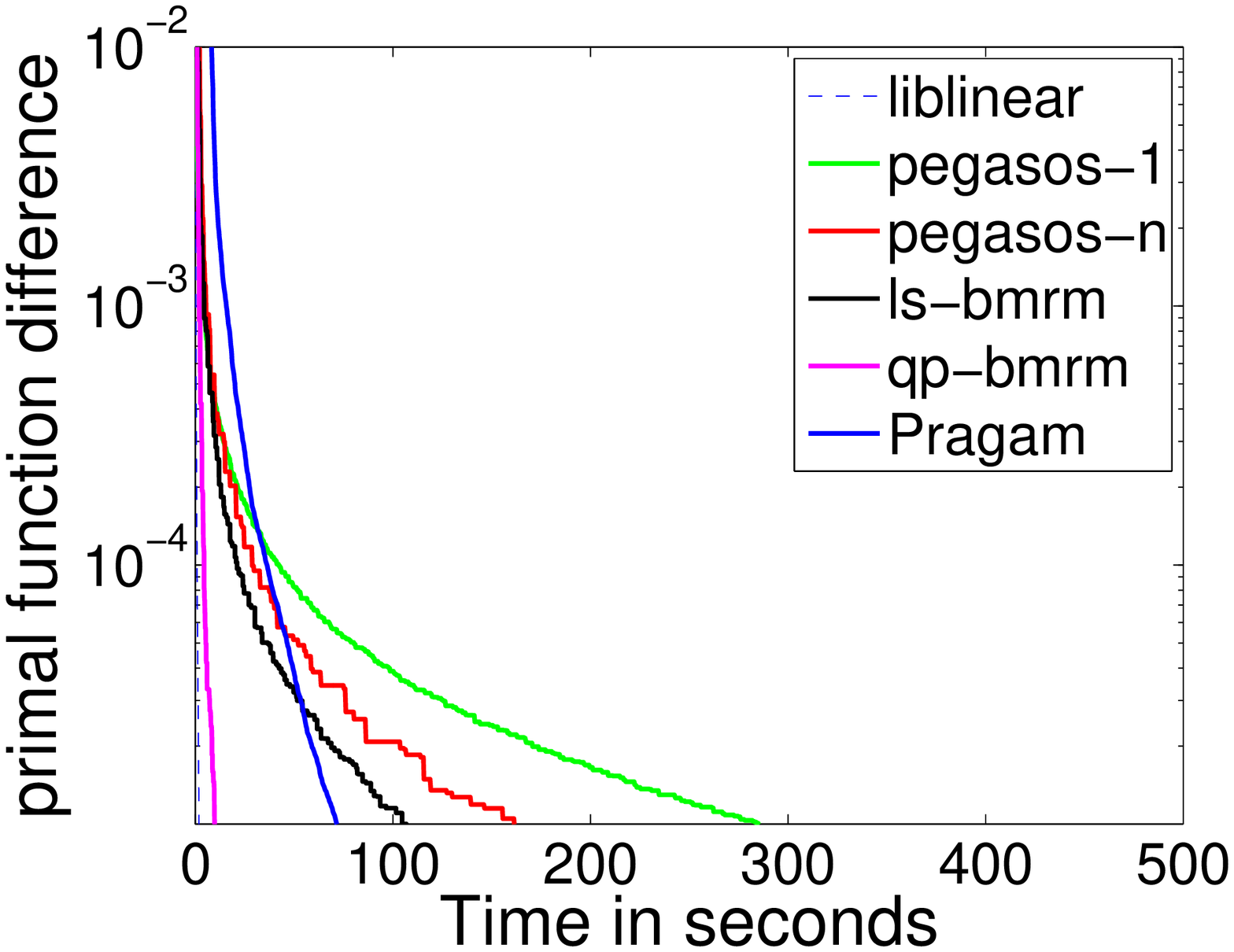}} \\
\vspace{-0.5em}
\caption{Primal function error versus time.}
\label{fig:fmin_diff_vs_time_text}
\end{centering}
\vspace{-1em}
\end{figure}

Due to lack of space, the figures of the detailed results are available in the Appendix \ref{app:exp}, and the main text only presents the results on three datasets.

We first compared how fast $\err_t := \min_{t' < t} J(\wb_{t'}) - J(\wb^*)$ decreases with
respect to the iteration index $t$.  We used $\err_t $ instead of $J(\wb_t) - J(\wb^*)$ because
$J(\wb_t)$ in \pegasos\ and \lsbmrm\ fluctuates drastically in some datasets.  The results in
Figure \ref{fig:fmin_diff_vs_iter_text} show \nest\ converges faster than \lsbmrm\ and \pegan\ which
both have $1/\ve$ rates.  \liblinear\ converges much faster than the rest algorithms, and
\qpbmrm\ is also fast.  \pegaone\ is not included because it converges very slowly in terms of
iterations.

Next, we compared in Figure \ref{fig:fmin_diff_vs_time_text} how fast $\err_t$ decreases in wall clock
time.  \nest\ is not fast in decreasing $\err_t$ to low accuracies like $10^{-3}$.  But it becomes
quite competitive when higher accuracy is desired, whereas \lsbmrm\ and \pegaone\ often take a
long time in this case.  Again, \liblinear\ is much faster than the other algorithms.

Another evaluation is on how fast a solver finds a model with reasonable accuracy.  At iteration
$t$, we examined the test accuracy of $\wb_{t'}$ where $t' := \argmin_{t' \le t} J(\wb_{t'})$, and
the result is presented in Figures \ref{fig:te_acc_vs_iter_text} and \ref{fig:te_acc_vs_time_text} with
respect to number of iterations and time respectively.  It can be seen that although \nest\
manages to minimize the primal function fast, its generalization power is not improved efficiently.
This is probably because this generalization performance hinges on the sparsity of the solution
(or number of support vectors, \cite{GraHerSha00}), and compared with all the other algorithms
\nest\ does not achieve any sparsity in the process of optimization. Asymptotically, all the
solvers achieve very similar testing accuracy.

Since the objective function of \nestb\ has a different feasible region than other optimizers which
do not use bias, we only compared its test accuracy.  In Figures \ref{fig:te_acc_vs_iter_text} and
\ref{fig:te_acc_vs_time_text}, the test accuracy of the optimal solution found by \nestb\ is always higher
than or similar to that of the other solvers.  In most cases, \nestb\ achieves the same test accuracy
faster than \nest\ both in number of iterations and time.

\begin{figure}[t]
\begin{centering}
\subfloat[astro-ph]{
    \includegraphics[width=4.35cm]{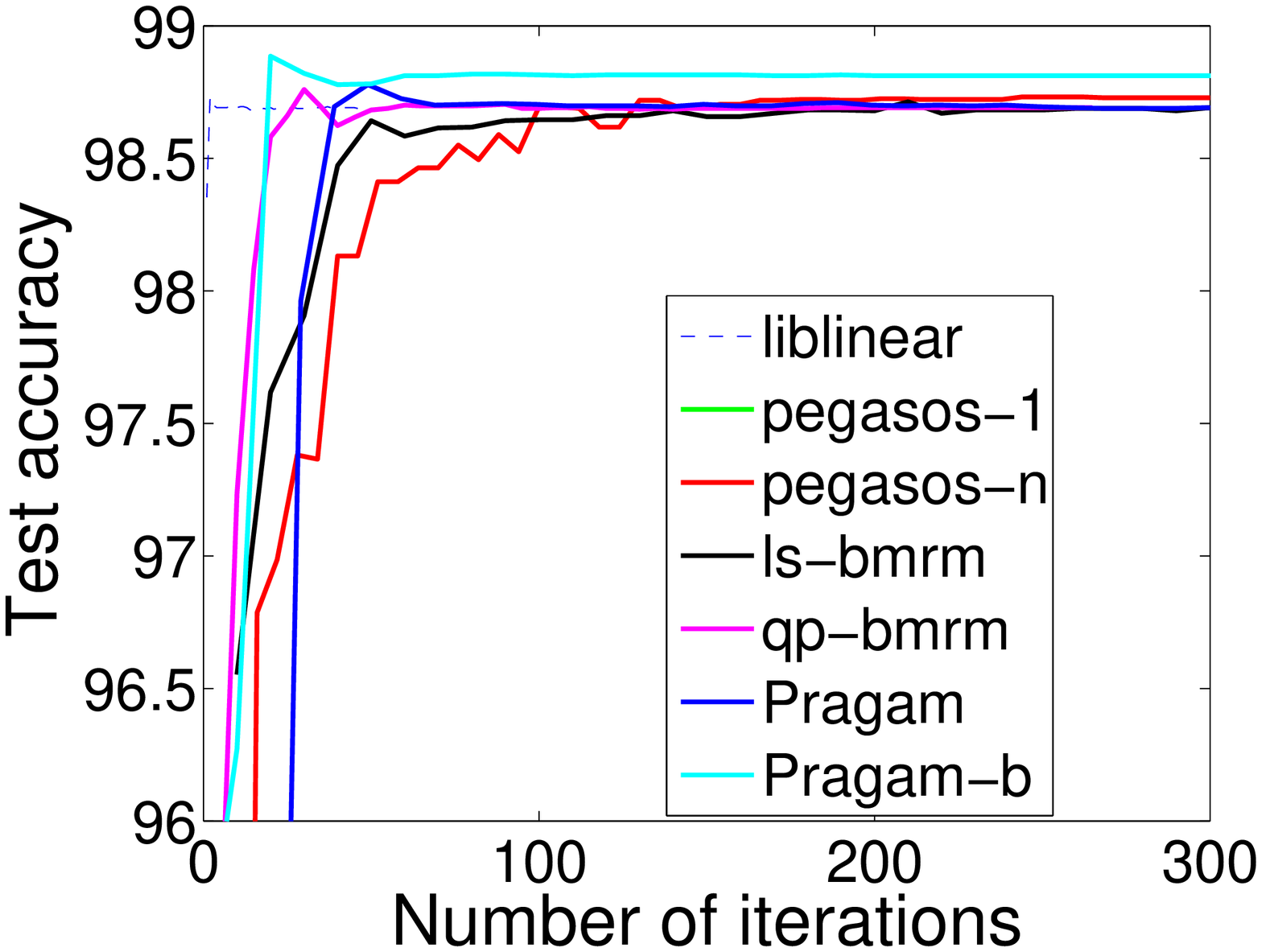}} ~~
\subfloat[news20]{
    \includegraphics[width=4.35cm]{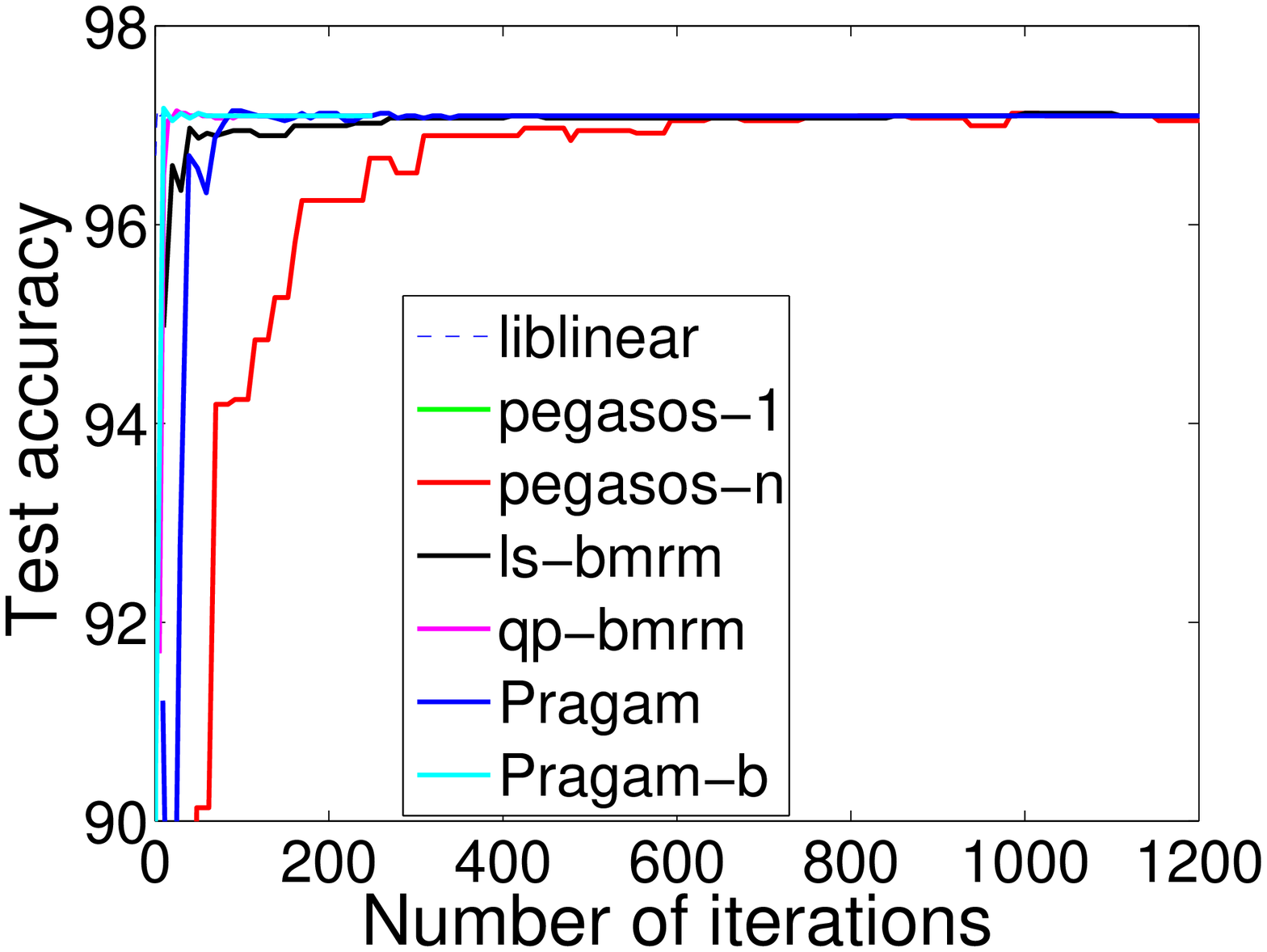}} ~~
\subfloat[real-sim]{
    \includegraphics[width=4.35cm]{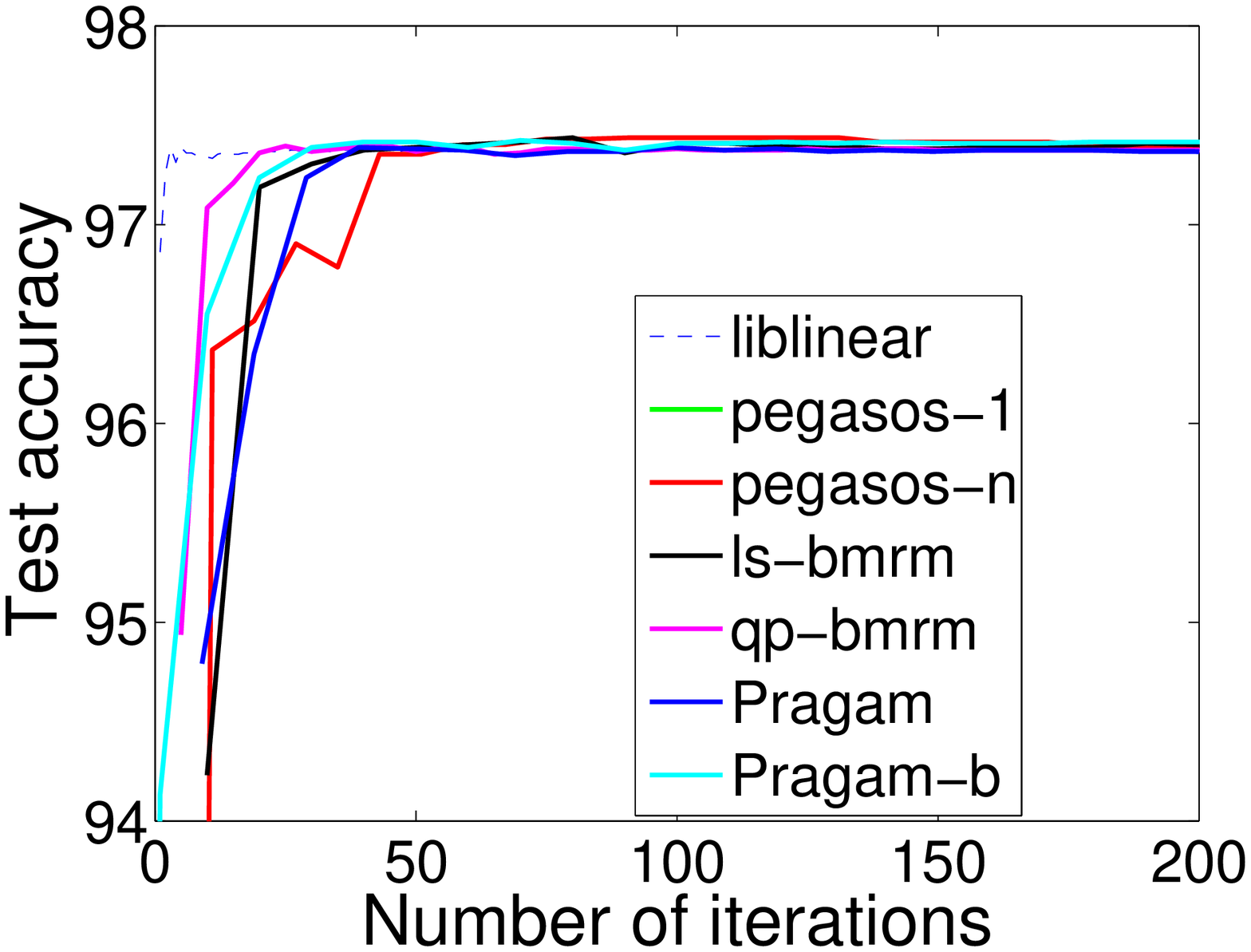}} \\
\caption{Test accuracy versus number of iterations.}
\label{fig:te_acc_vs_iter_text}
\end{centering}
\end{figure}
\begin{figure}[t]
\begin{centering}
\subfloat[astro-ph]{
    \includegraphics[width=4.35cm]{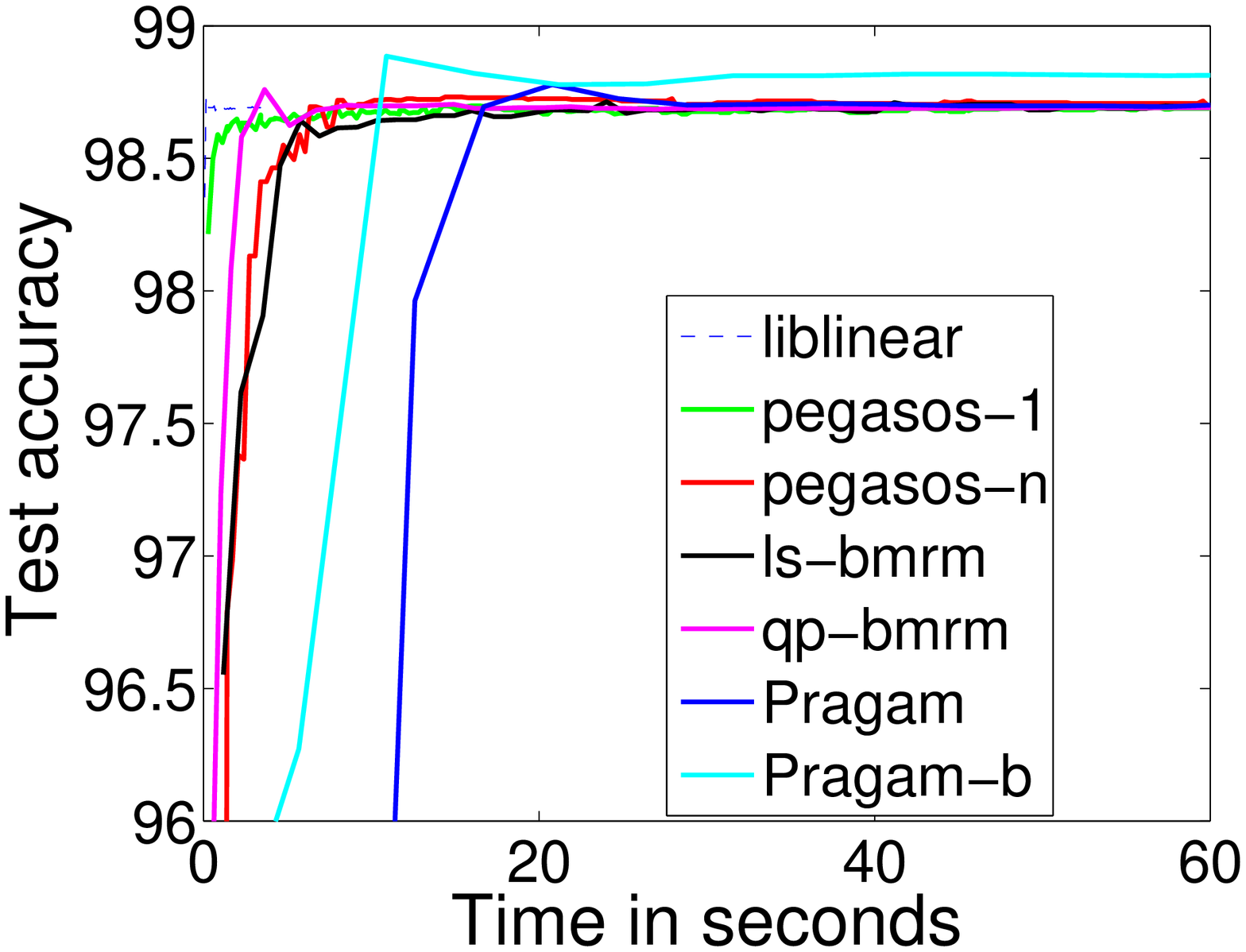}} ~~
\subfloat[news20]{
    \includegraphics[width=4.35cm]{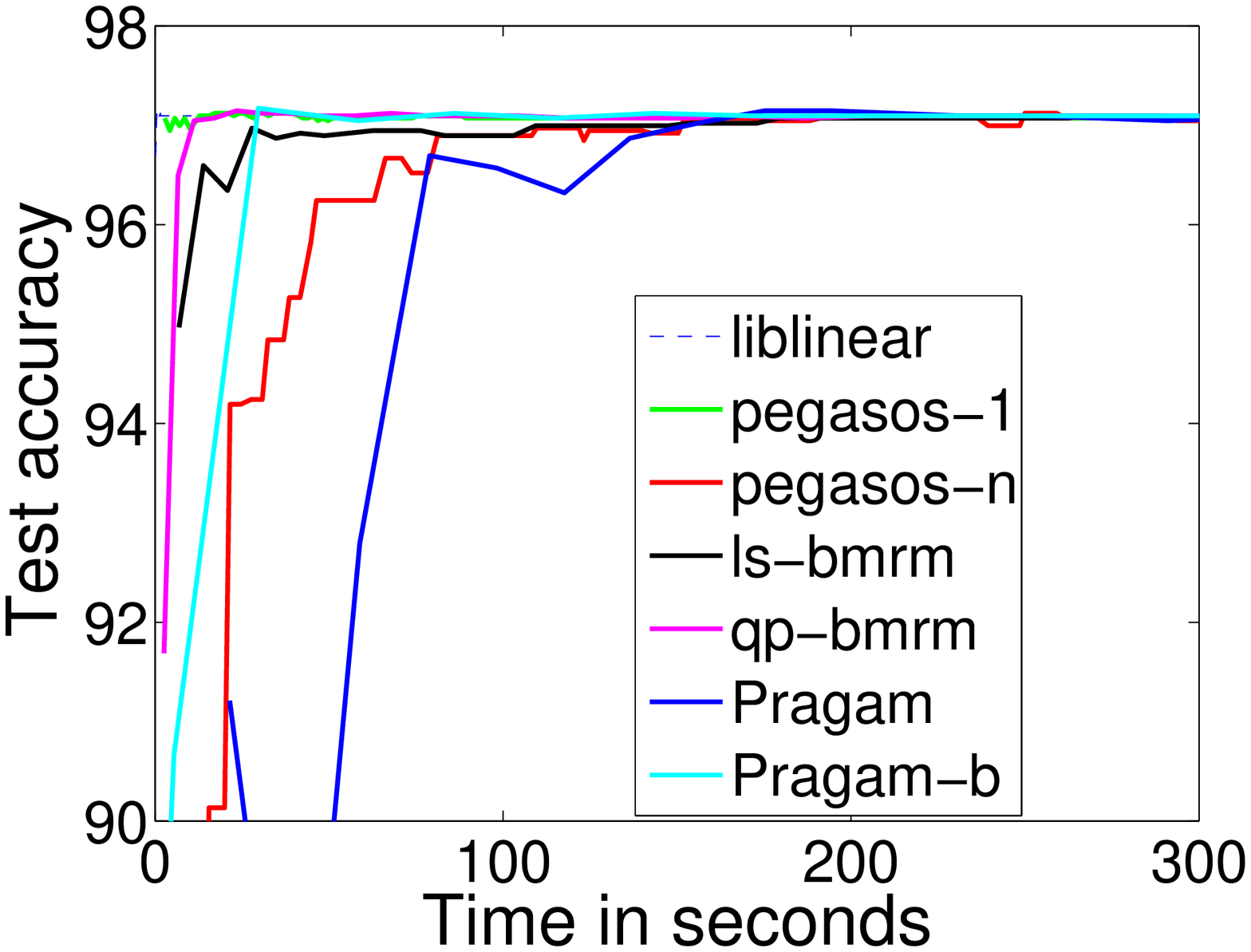}} ~~
\subfloat[real-sim]{
    \includegraphics[width=4.35cm]{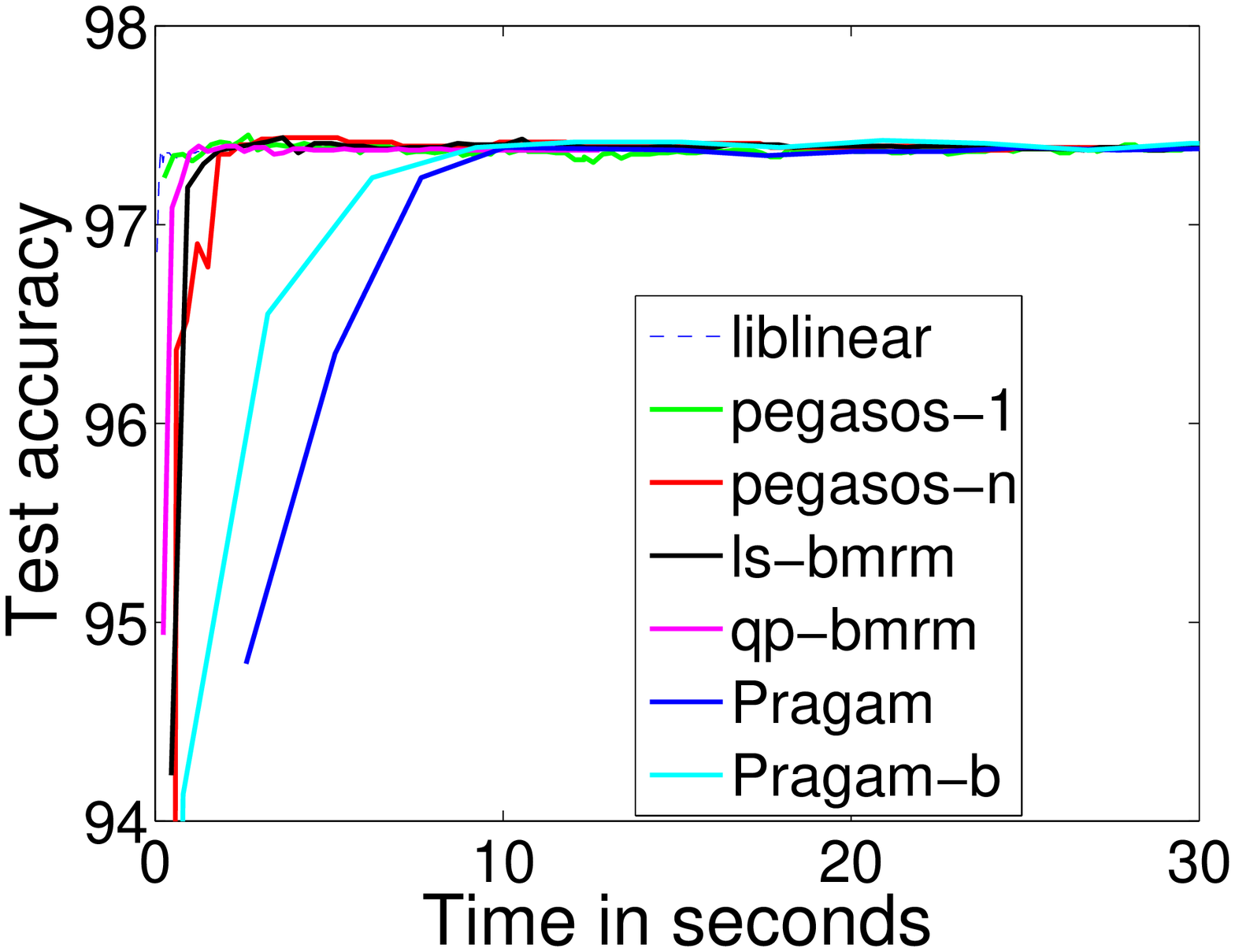}} \\
\caption{Test accuracy versus time.}
\label{fig:te_acc_vs_time_text}
\end{centering}
\end{figure}

\section{Discussion and Conclusions}
\label{sec:ResultsDiscussion}

In this paper we described a new lower bound for the number of
iterations required by BMRM and similar algorithms which are widely
used solvers for the regularized risk minimization problem. This shows that the
iteration bounds shown for these solvers is optimum. Our lower bounds
are somewhat surprising because the empirical performance of these
solvers indicates that they converge linearly to an $\ve$ accurate
solution on a large number of datasets. Perhaps a more refined analysis
is needed to explain this behavior.

The SVM problem has received significant research attention
recently. For instance, \cite{ShaSinSre07} proposed a stochastic
subgradient algorithm Pegasos. The convergence of Pegasos is analyzed in
a stochastic setting and it was shown that it converges in $O(1/\ve)$
iterations. We believe that our lower bounds can be extended to any
arbitrary subgradient based solvers in the primal including
Pegasos. This is part of ongoing research.

Our technique of solving the dual optimization problem is not new. A
number of solvers including SVM-Light \cite{Joachims99} and SMO
\cite{Platt99} work on the dual problem. Even though linear convergence
is established for these solvers, their rates have $n^{\geq 3}$
dependence which renders the analysis unusable for practical purposes.
Other possible approaches include the interior-point method of
\cite{FerMun02} which costs $O(nd^2 \log (\log (1/\ve)))$ time and
$O(d^2)$ space where $d$ refers to the dimension of the features. LibLinear
\cite{HsiChaLinKeeetal08} performs coordinate descent in the dual, and has
$O(nd \log (1/\ve))$ complexity but only after more than $O(n^2)$ steps.
Mirror descent algorithms \cite{BecTeb03} cost $O(nd)$ per iteration,
but their convergence rate is $1/\ve^2$. These rates are prohibitively
expensive when $n$ is very large.

The $O(1/\sqrt{\ve})$ rates for the new SVM algorithm we described in
this paper has a favorable dependence on $n$ as well as
$\lambda$. Although our emphasis has been largely theoretical, the
empirical experiments indicate that our solver is competitive with the
state of the art. Finding an efficient solver with fast rates of
convergence and good empirical performance remains a holy grail of
optimization for machine learning.


\bibliographystyle{unsrt}
\bibliography{../../../bibfile}

\newpage
\appendix
\section*{Appendix}
\section{Minimax Theorem on Convex spaces}
\label{sec:MinimaxTheoremConvex}

The reversal of $\min$ and $\max$ operators in \eqref{eq:minimax} follows from
the following theorem. (Theorem 3 in \cite{Park98})
\begin{theorem}
  Let $X$ be a convex space, $Y$ a Hausdorff compact convex space and $f:X
  \times Y \to Z$ a function. Suppose that
  \begin{itemize}
  \item there is a subset $U \subset Z$ such that $ a, b \in f(X \times Y)$ with
    $a<b$ implies $U \cap (a,b) \neq \emptyset$;
  \item $f(x,\cdot)$ is lower semicontinuous ($l.s.c.$) on $Y$ and $\cbr{ y \in Y:
    f(x,y)<s}$ is convex for each $x \in X$ and $s \in U$ and
   \item  $f(\cdot, y)$ is upper semicontinuous ($u.s.c.$) on $X$ and $\cbr{ x \in X:
     f(x,y) <s}$ is convex for each $y \in Y$ and $s \in U$
  \end{itemize}
  Then
  \begin{align*}
    \max_{x \in X}\min_{y \in Y}f(x,y) = \min_{y \in Y}\max_{x \in X}f(x,y)
  \end{align*}
\end{theorem}
It is trivial to show that our setting satisfies the above three conditions for the
linear form $f(\wb,\alphab) = \wb^{\top}\Ab_t\alphab$.

In our case $Z = \RR$. Since $\alpha \in \Delta_t$ and the columns of $\Ab$ have
Euclidean norm 1, it can be easily shown that $w$ is bounded in a ball of radius
$1/\lambda$. Thus the range space is a finite subset of $\RR$ and we choose $U$
to be the entire range space so that it will satisfy the first condition as
$f(\wb,\alphab)$ is a continuous function.

Also $g(\alphab) = f(\wb,\alphab)$ and $h(\wb) = f(\wb, \alphab)$ are continuous
functions in $\alphab$ and $\wb$ respectively and are thus by definition lower (upper)
semicontinuous in $\alphab$ ($\wb$). The convexity of the sets in the $2^{nd}$ and
$3^{rd}$ condition follows from first principles using definition of convexity. Thus
we can use the minimax theorem to obtain \eqref{eq:minimax}.

\section{A linear time algorithm for a box constrained diagonal QP with a single linear equality constraint}
\label{sec:simple_qp}
It can be shown that the dual optimization problem $D(\alphab)$ from
\eqref{eq:primal_dual_svm} can be reduced into a box constrained QP with a single linear
equality constraint.

In this section, we focus on the following simple QP:
\begin{align*}
  \min \frac{1}{2} \sum_{i=1}^n &d_i^2 (\alpha_i - m_i)^2 \\
  s.t. \qquad l_i \le &\alpha_i \le u_i \quad \forall i \in [n];  \\
  \sum_{i=1}^n & \sigma_i \alpha_i = z.
\end{align*}
Without loss of generality, we assume $l_i < u_i$ and $d_i \neq 0$ for all $i$.
Also assume $\sigma_i \neq 0$ because otherwise $\alpha_i$ can be solved independently.
  To make the feasible region nonempty, we also assume
\[
\sum_i \sigma_i (\delta(\sigma_i > 0) l_i + \delta(\sigma_i < 0) u_i)
\le z \le
\sum_i \sigma_i (\delta(\sigma_i > 0) u_i + \delta(\sigma_i < 0) l_i).
\]
The algorithm we describe below stems from \cite{ParKov90} and finds the exact optimal
solution in $O(n)$ time, faster than the $O(n \log n)$ complexity in
\cite{DucShaSigCha08}.

With a simple change of variable $\beta_i = \sigma_i (\alpha_i - m_i)$, the problem is
simplified as
\begin{minipage}[t]{6cm}
\begin{align*}
  \min \qquad \frac{1}{2} \sum_{i=1}^n & \dbar^{2}_i \beta_i^2 \\
  s.t. \qquad l'_i \le & \beta_i \le u'_i \quad \forall i \in [n];  \\
  \sum_{i=1}^n & \beta_i = z',
\end{align*}
\end{minipage}
\begin{minipage}[t]{1.2cm}
\vspace{4em}
where
\end{minipage}
\begin{minipage}[t]{6cm}
\begin{align*}
  l'_i &= \left\{ {\begin{array}{ll}
   \sigma_i (l_i - m_i) & \text{if } \sigma_i > 0  \\
   \sigma_i (u_i - m_i) & \text{if } \sigma_i < 0  \\
\end{array}} \right., \\
  u'_i &= \left\{ {\begin{array}{ll}
   \sigma_i (u_i - m_i) & \text{if } \sigma_i > 0  \\
   \sigma_i (l_i - m_i) & \text{if } \sigma_i < 0  \\
\end{array}} \right., \\
\dbar^2_i &= \frac{d_i^2}{\sigma_i^2}, \quad z' = z - \sum_i \sigma_i m_i.
\end{align*}
\end{minipage}

We derive its dual via the standard Lagrangian.
\begin{align*}
  L = \frac{1}{2} \sum_i \dbar^{2}_i \beta_i^2 - \sum_i \rho_i^+ (\beta_i
  - l'_i) + \sum_i \rho_i^- (\beta_i - u'_i) - \lambda \left(\sum_i \beta_i
  - z' \right).
\end{align*}
Taking derivative:
\begin{align}
\label{eq:dual_connect_simple_qp}
  \frac{\partial L}{\partial \beta_i} = \dbar^{2}_i \beta_i - \rho_i^+
  + \rho_i^- - \lambda = 0
  \quad \Rightarrow \quad
  \beta_i = \dbar^{-2}_i (\rho_i^+ - \rho_i^- + \lambda).
\end{align}
Substituting into $L$, we get the dual optimization problem
\begin{align*}
  \min D(\lambda, \rho_i^+, \rho_i^-) &= \frac{1}{2} \sum_i \dbar^{-2}_i
  (\rho_i^+ - \rho_i^- + \lambda)^2 - \sum_i \rho_i^+ l'_i + \sum_i
  \rho_i^+ u'_i - \lambda z' \\
  s.t. \qquad &\rho_i^+ \ge 0, \quad \rho_i^- \ge 0  \quad \forall i \in [n].
\end{align*}
Taking derivative of $D$ with respect to $\lambda$, we get:
\begin{align}
\label{eq:lambda_constraint_simple_qp}
  \sum_i \dbar^{-2}_i (\rho_i^+ - \rho_i^- + \lambda) - z' = 0.
\end{align}
The KKT condition gives:
\begin{subequations}
\label{eq:kkt_simple_qp}
\begin{align}
\label{eq:kkt_1_simple_qp}
  \rho_i^+ (\beta_i - l'_i) &= 0, \\
\label{eq:kkt_2_simple_qp}
  \rho_i^- (\beta_i - u'_i) &= 0.
\end{align}
\end{subequations}
Now we enumerate four cases.
\paragraph{1. $\rho_i^+ > 0$, $\rho_i^- > 0$.}  This implies that $l'_i =
\beta_i = u'_i$, which is contradictory to our assumption.
\paragraph{2. $\rho_i^+ = 0$, $\rho_i^- = 0$.}  Then by
\eqref{eq:dual_connect_simple_qp}, $\beta_i = \dbar^{-2}_i \lambda \in
      [l'_i, u'_i]$, hence $\lambda \in [\dbar^{2}_i l'_i, \dbar^{2}_i u'_i]$.
\paragraph{3. $\rho_i^+ > 0$, $\rho_i^- = 0$.}  Now by \eqref{eq:kkt_simple_qp}
and \eqref{eq:dual_connect_simple_qp}, we have $l'_i = \beta_i = \dbar^{-2}_i
(\rho_i^+ + \lambda) > \dbar^{-2}_i \lambda$, hence $\lambda < \dbar^{2}_i l'_i$
and $\rho_i^+ = \dbar^{2}_i l'_i - \lambda$.
\paragraph{4. $\rho_i^+ = 0$, $\rho_i^- > 0$.}  Now by \eqref{eq:kkt_simple_qp}
and \eqref{eq:dual_connect_simple_qp}, we have $u'_i = \beta_i = \dbar^{-2}_i
(-\rho_i^- + \lambda) < \dbar^{-2}_i \lambda$, hence $\lambda > \dbar^{2}_i u'_i$
and $\rho_i^- = -\dbar^{2}_i u'_i + \lambda$.

In sum, we have $\rho_i^+ = [\dbar^{2}_i l'_i - \lambda]_+$ and $\rho_i^- =
[\lambda - \dbar^{2}_i u'_i]_+$.  Now \eqref{eq:lambda_constraint_simple_qp} turns
into
\begin{align}
\label{eq:lambda_find_root}
  f(\lambda) := \sum_i \underbrace{\dbar^{-2}_i ([\dbar^{2}_i l'_i - \lambda]_+ -
    [\lambda - \dbar^{2}_i u'_i]_+ + \lambda)}_{=: h_i(\lambda)} - z' = 0.
\end{align}
In other words, we only need to find the root of $f(\lambda)$ in
\eqref{eq:lambda_find_root}.  $h_i(\lambda)$ is plotted in Figure \ref{fig:hi_lambda}.
Note that $h_i(\lambda)$ is a monotonically increasing function of $\lambda$, so the
whole $f(\lambda)$ is monotonically increasing in $\lambda$.  Since $f(\infty) \ge 0$
by $z' \le \sum_i u'_i$ and $f(-\infty) \le 0$ by $z' \ge \sum_i l'_i$, the root must
exist.  Considering that $f$ has at most $2n$ kinks (nonsmooth points) and is linear
between two adjacent kinks, the simplest idea is to sort $\cbr{\dbar^{2}_i l'_i,
\dbar^{2}_i u'_i : i \in [n]}$ into $s^{(1)} \le \ldots \le s^{(2n)}$.  If $f(s^{(i)})$
and $f(s^{(i+1)})$ have different signs, then the root must lie between them and can be
easily found because $f$ is linear in $[s^{(i)}, s^{(i+1)}]$.  This algorithm takes at
least $O(n \log n)$ time because of sorting.

\begin{figure}
\centering
    \includegraphics[height=3cm]{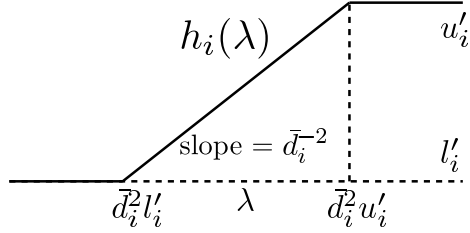}
\caption{$h_i(\lambda)$}
\label{fig:hi_lambda}
\end{figure}

However, this complexity can be reduced to $O(n)$ by making use of the fact that the
median of $n$ (unsorted) elements can be found in $O(n)$ time.  Notice that due to the
monotonicity of $f$, the median of a set $S$ gives exactly the median of function values,
\ie, $f(\MED(S)) = \MED(\cbr{f(x):x \in S})$.  Algorithm \ref{algo:linear_simple_qp}
sketches the idea of binary search.  The while loop terminates in $\log_2 (2n)$ iterations
because the set $S$ is halved in each iteration.  And in each iteration, the time
complexity is linear to $|S|$, the size of current $S$.  So the total complexity is $O(n)$.
Note the evaluation of $f(m)$ potentially involves summing up $n$ terms as in
\eqref{eq:lambda_find_root}.  However by some clever aggregation of slope and offset, this
can be reduced to $O(|S|)$.

\begin{algorithm}[t]
\begin{algorithmic}[1]
    \caption{\label{algo:linear_simple_qp} $O(n)$ algorithm to find the root of $f(\lambda)$.
      Ignoring boundary condition checks.}
    \STATE Set kink set $S \leftarrow \cbr{\dbar_i^2 l'_i : i \in [n]} \cup \cbr{\dbar_i^2
      u'_i: i \in [n]}$.
    \WHILE{$\abr{S} > 2$}
        \STATE Find median of $S$: $m \leftarrow \MED(S)$.
        \IF{$f(m) \ge 0$}
            \STATE $S \leftarrow \cbr{x \in S : x \le m}$.
        \ELSE
            \STATE $S \leftarrow \cbr{x \in S : x \ge m}$.
        \ENDIF
    \ENDWHILE
    \STATE Return root $\frac{l f(u) - u f(l)}{f(u) - f(l)}$ where $S = \cbr{l,u}$.
\end{algorithmic}
\end{algorithm}

\section{Derivation of $g^{\star}(\alphab)$ }
\label{sec:Derivgstar}

To see $g^{\star}(\alphab) = \min_{b \in \RR} \frac{1}{n} \sum_i
\sbr{1 + \alpha_i - y_i b}_+$ in \eqref{eq:primal_dual_svm}, it suffices to show that for all $\alphab
\in \RR^n$:
\begin{align}
  \label{eq:fenchel_dual_bias_SVM}
  \sup_{\rhob \in Q_2} \inner{\rhob}{\alphab} + \sum_i \rho_i = \min_{b \in \RR}
  \frac{1}{n} \sum_i \sbr{1 + \alpha_i - y_i b}_+.
\end{align}
Posing the latter optimization as:
\[
\min_{\xi_i, b} \frac{1}{n} \sum_i \xi_i \qquad s.t. \quad 1 + \alpha_i - y_i b \le
\xi_i, \quad \xi_i \ge 0.
\]
Write out the Lagrangian:
\[
L = \frac{1}{n} \sum_i \xi_i + \sum_i \rho_i (1 + \alpha_i - y_i b - \xi_i)
- \sum_i \beta_i \xi_i.
\]
Taking partial derivatives:
\begin{align*}
  \frac{\partial L}{\partial \xi_i} &= \frac{1}{n} - \rho_i - \beta_i = 0 \qquad
  &\Rightarrow& \qquad \rho_i \in [0, n^{-1}], \\
  \frac{\partial L}{\partial b} &= -\sum_i \rho_i y_i = 0 \qquad &\Rightarrow&
  \qquad \sum_i \rho_i y_i = 0.
\end{align*}
Plugging back into $L$,
\[
L = \sum_i \rho_i (1 + \alpha_i), \qquad \text{s.t. } \quad \rho_i \in [0, n^{-1}],
\quad \sum_i \rho_i y_i = 0.
\]
Maximizing $L$ wrt $\rhob$ is exactly the LHS of \eqref{eq:fenchel_dual_bias_SVM}.

\section{Experimental Results in Detail}
\label{app:exp}

The $\lambda$s used in the experiment are:

\begin{table}[htb]
\centering
  \begin{tabular}{|ll|ll|ll|ll|}
  \hline
  dataset & $\lambda$ &  dataset & $\lambda$ &  dataset & $\lambda$ &  dataset & $\lambda$ \\
  \hline
    adult & $2^{-18}$ & astro-ph &  $2^{-17}$ & aut-avn & $2^{-17}$ & covertype & $2^{-17}$ \\
    \hline
    news20 & $2^{-14}$ & reuters-c11 & $2^{-19}$ & reuters-ccat & $2^{-19}$ & real-sim & $2^{-16}$ \\
    \hline
    web8 & $2^{-17}$ &&&&&& \\
    \hline
  \end{tabular}
\end{table}

\begin{figure}[htbp]
\begin{centering}
\subfloat[adult9 (pegasos diverged)]{
    \includegraphics[width=4.35cm]{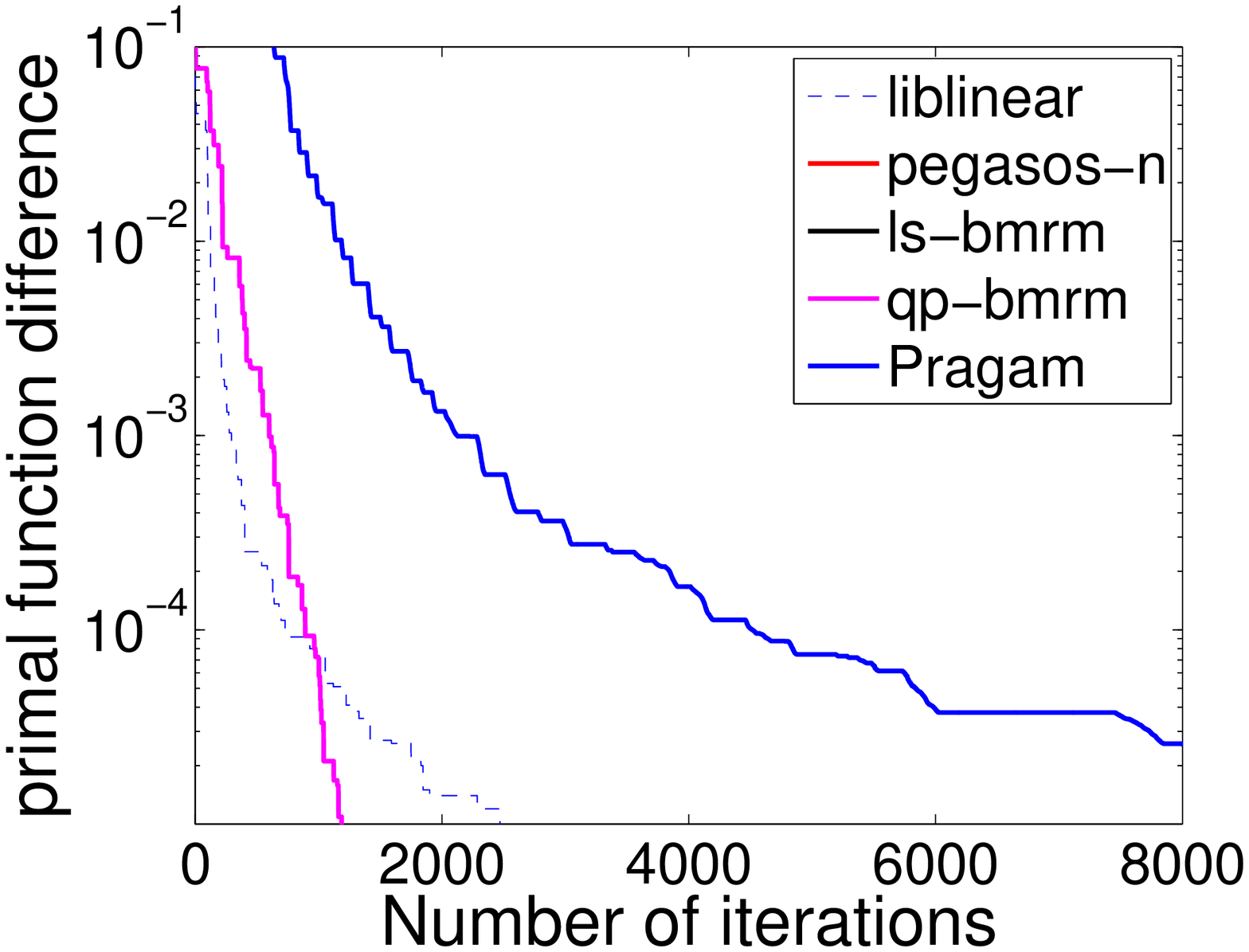}} ~~
\subfloat[astro-ph]{
    \includegraphics[width=4.35cm]{astro-ph_fmin_diff_vs_iter}} ~~
\subfloat[aut-avn]{
    \includegraphics[width=4.35cm]{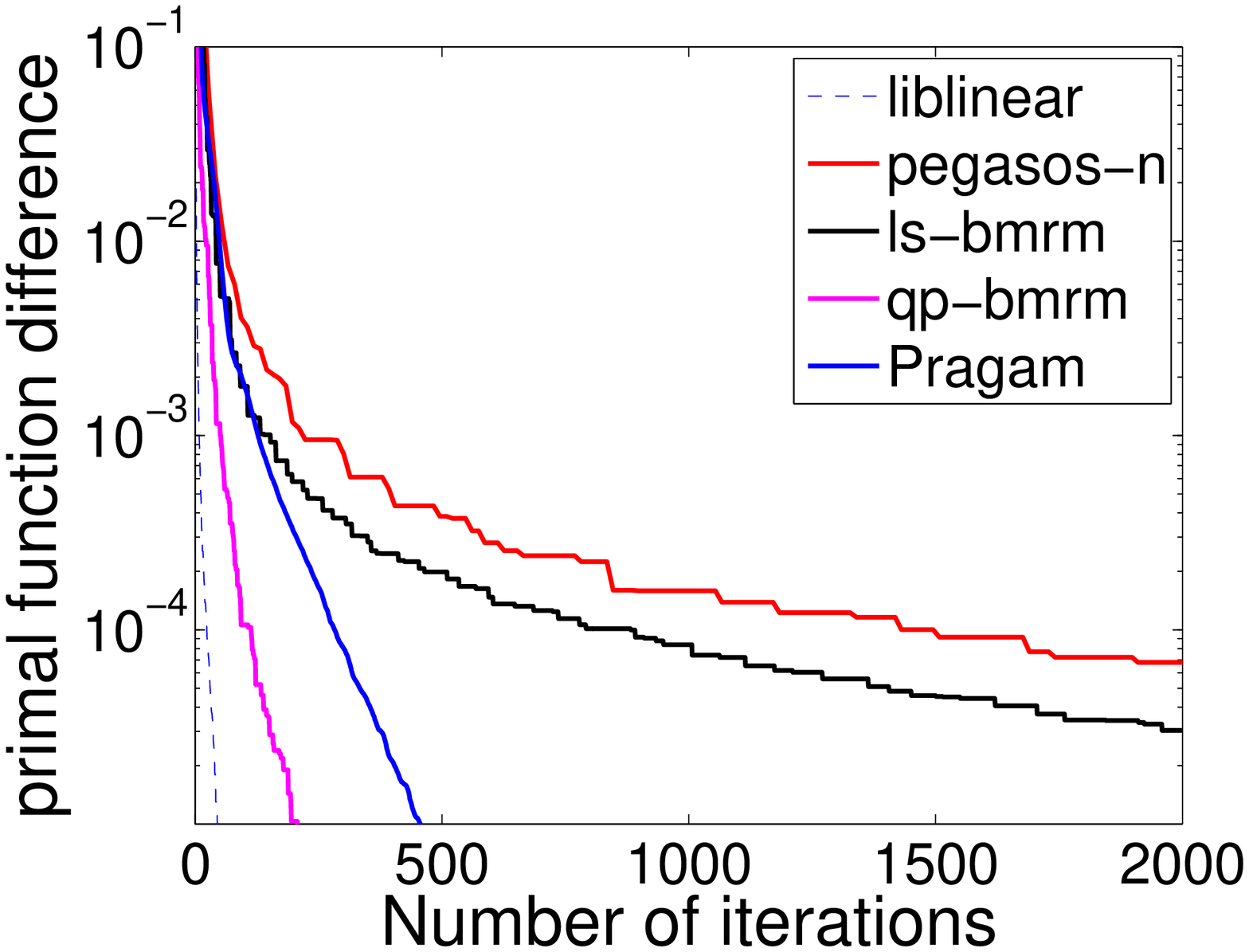}} \\
\subfloat[covertype]{
    \includegraphics[width=4.35cm]{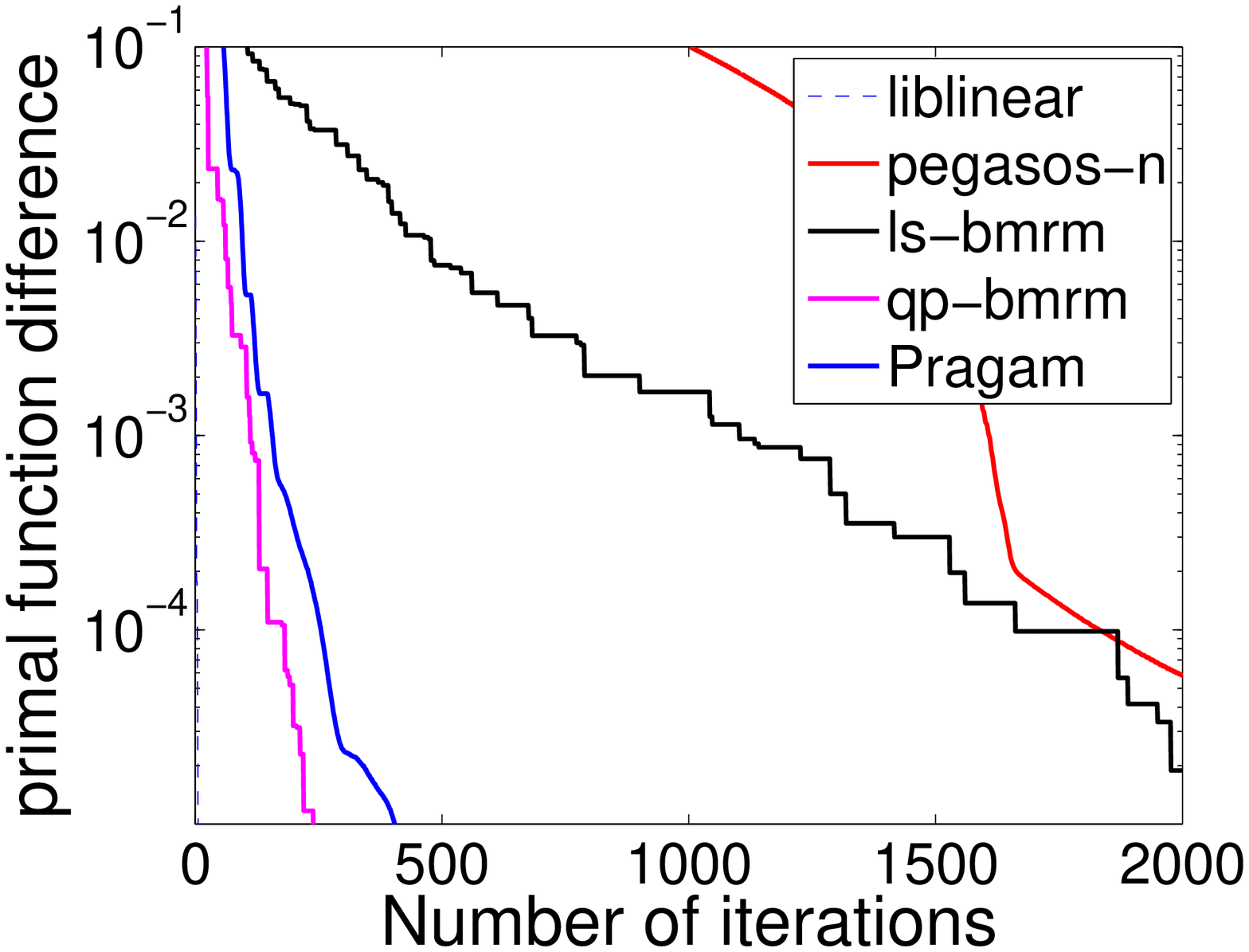}} ~~
\subfloat[news20]{
    \includegraphics[width=4.35cm]{news20_fmin_diff_vs_iter}} ~~
\subfloat[real-sim]{
    \includegraphics[width=4.35cm]{real-sim_fmin_diff_vs_iter}} \\
\subfloat[reuters-c11]{
    \includegraphics[width=4.35cm]{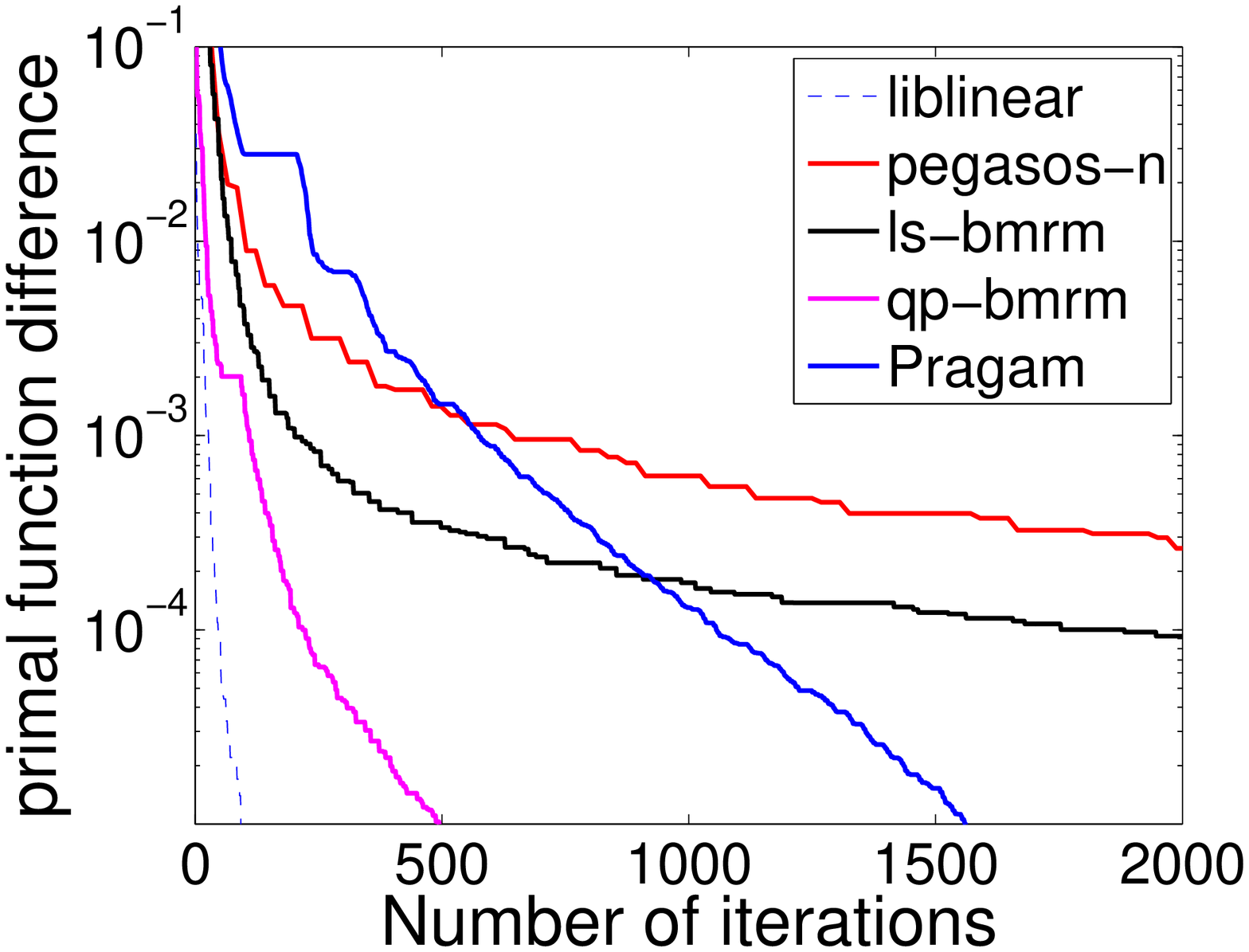}} ~~
\subfloat[reuters-ccat]{
    \includegraphics[width=4.35cm]{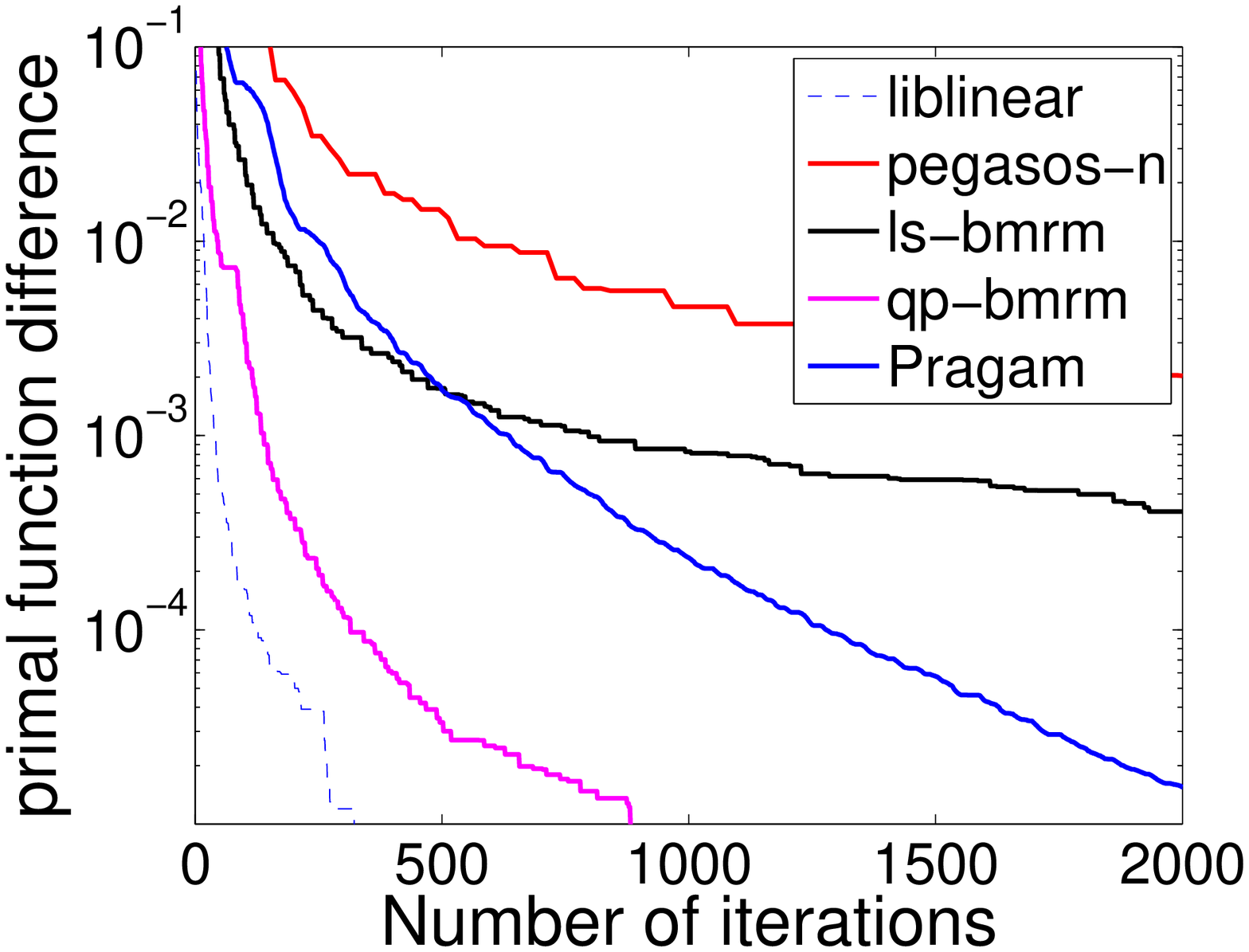}} ~~
\subfloat[web8]{
    \includegraphics[width=4.35cm]{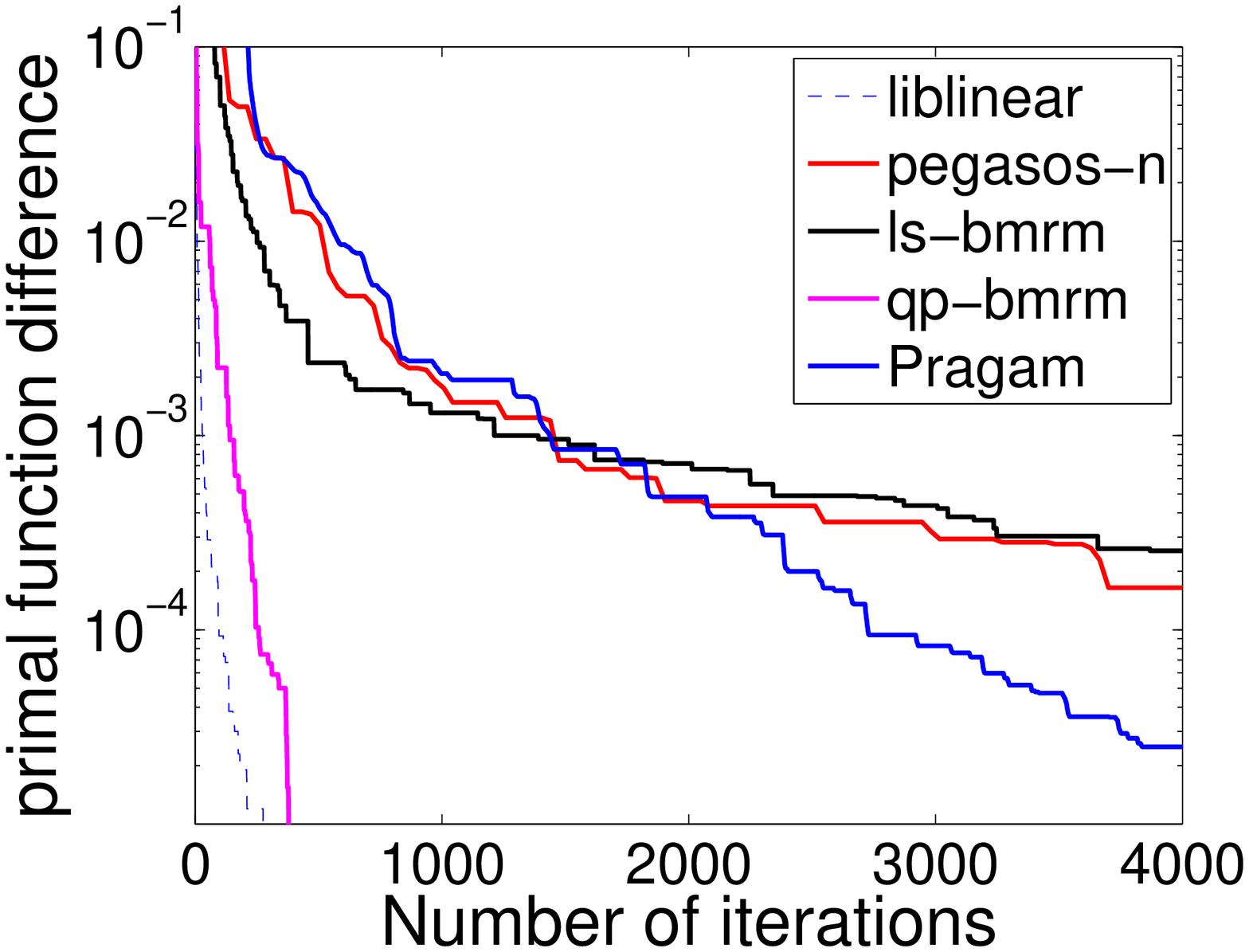}}
\caption{Primal function error versus number of iterations.}
\label{fig:fmin_diff_vs_iter_app}
\end{centering}
\end{figure}

\begin{figure}[htbp]
\begin{centering}
\subfloat[adult9 (pegasos diverged)]{
    \includegraphics[width=4.35cm]{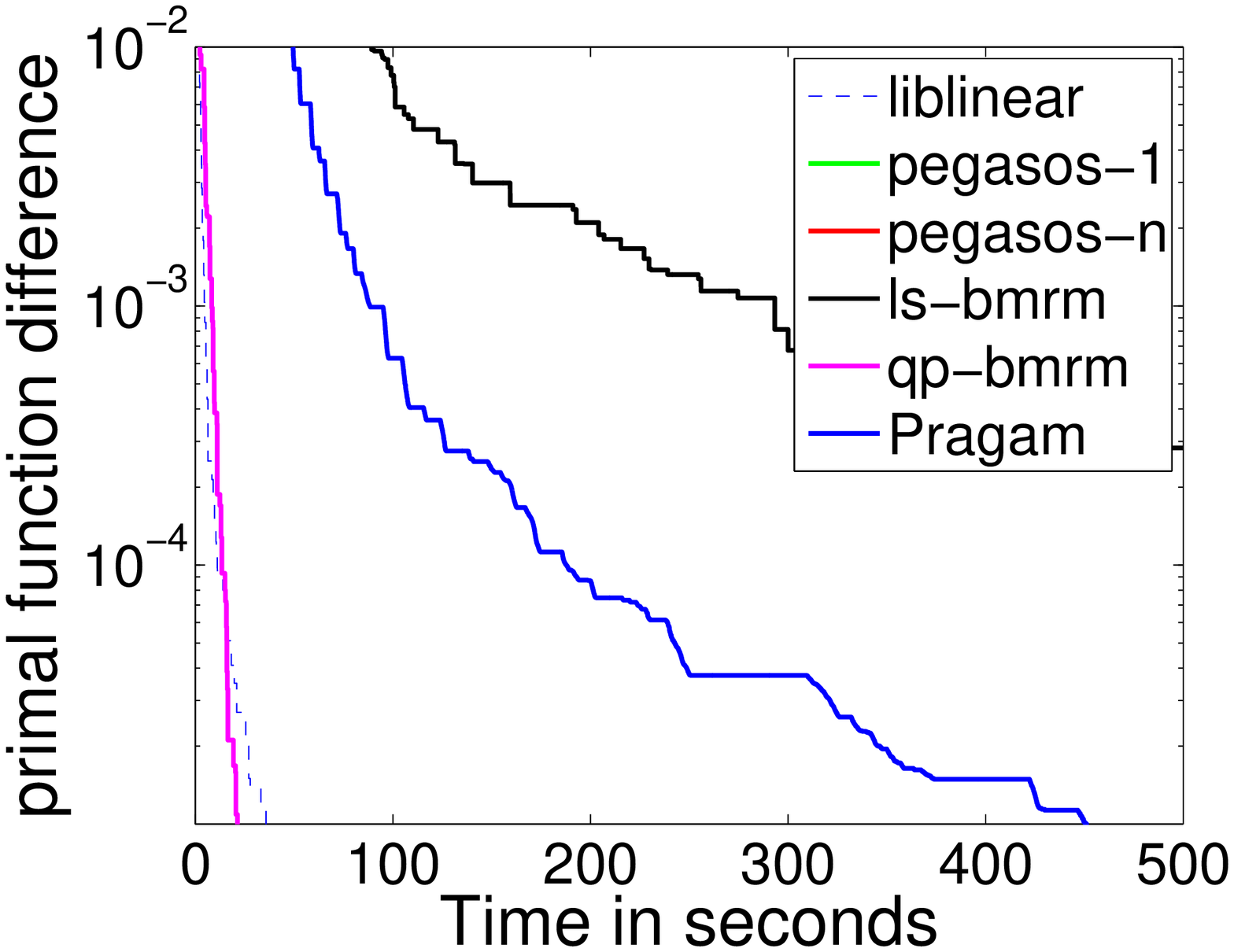}} ~~
\subfloat[astro-ph]{
    \includegraphics[width=4.35cm]{astro-ph_fmin_diff_vs_time}} ~~
\subfloat[aut-avn]{
    \includegraphics[width=4.35cm]{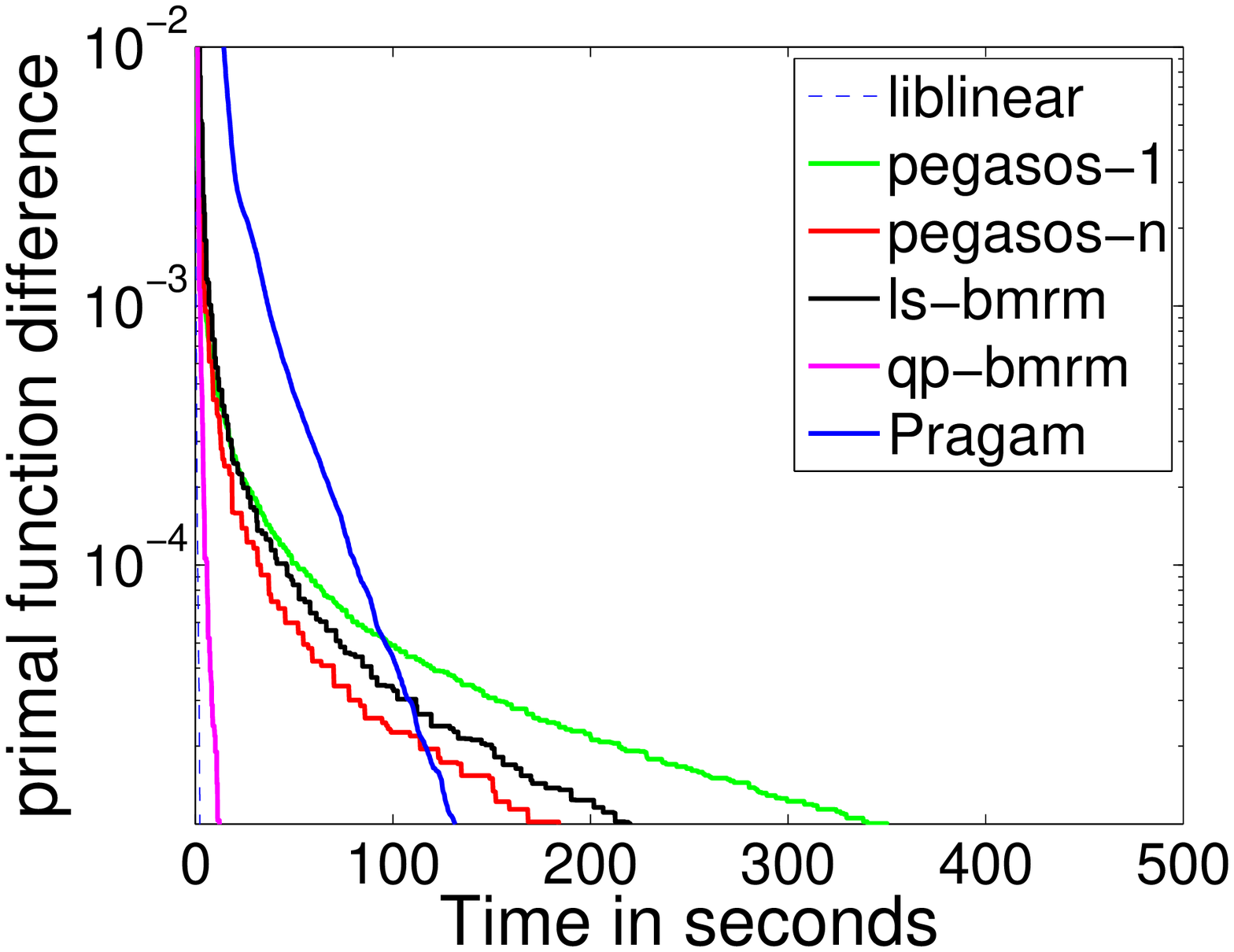}} \\
\subfloat[covertype]{
    \includegraphics[width=4.35cm]{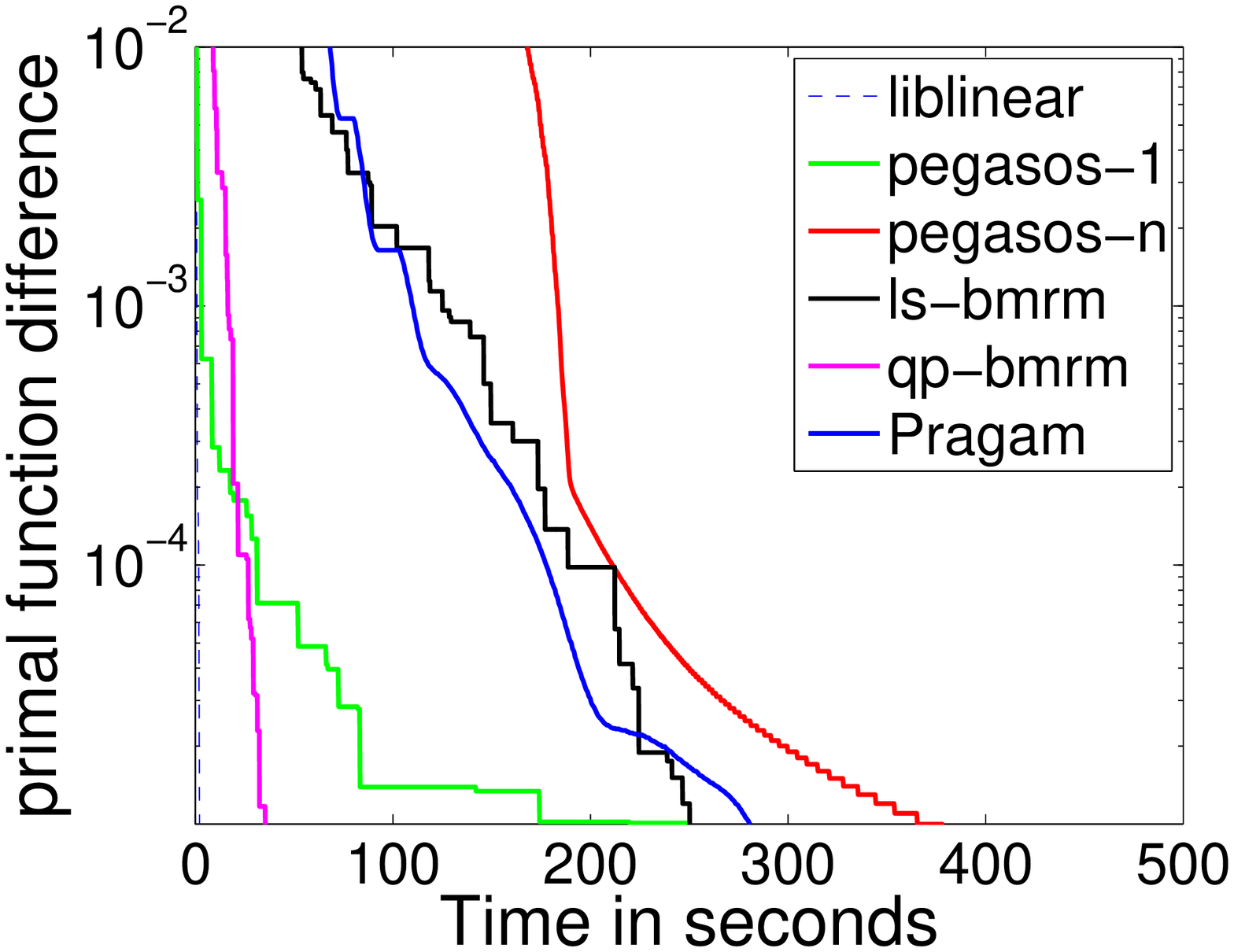}} ~~
\subfloat[news20]{
    \includegraphics[width=4.35cm]{news20_fmin_diff_vs_time}} ~~
\subfloat[real-sim]{
    \includegraphics[width=4.35cm]{real-sim_fmin_diff_vs_time}} \\
\subfloat[reuters-c11]{
    \includegraphics[width=4.35cm]{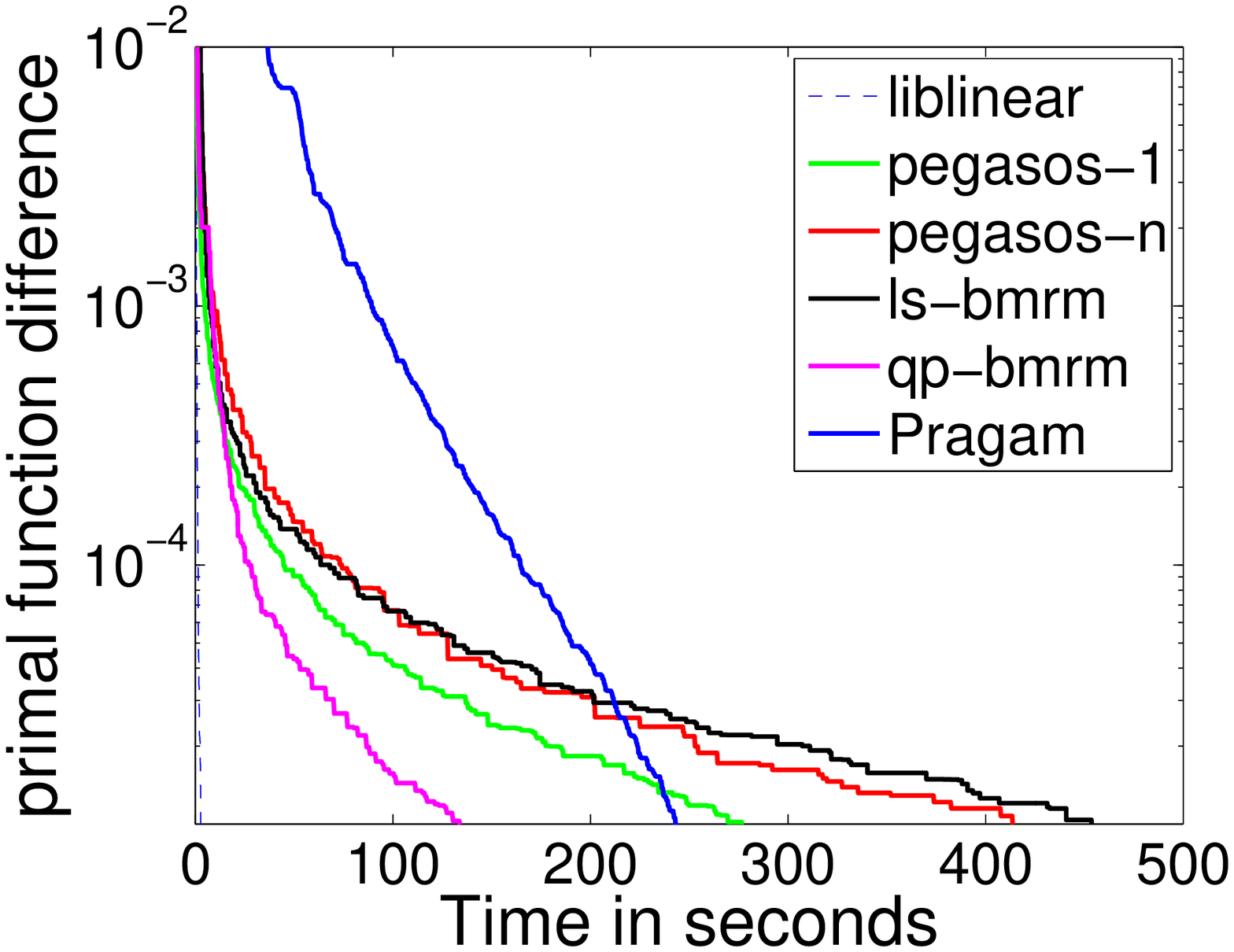}} ~~
\subfloat[reuters-ccat]{
    \includegraphics[width=4.35cm]{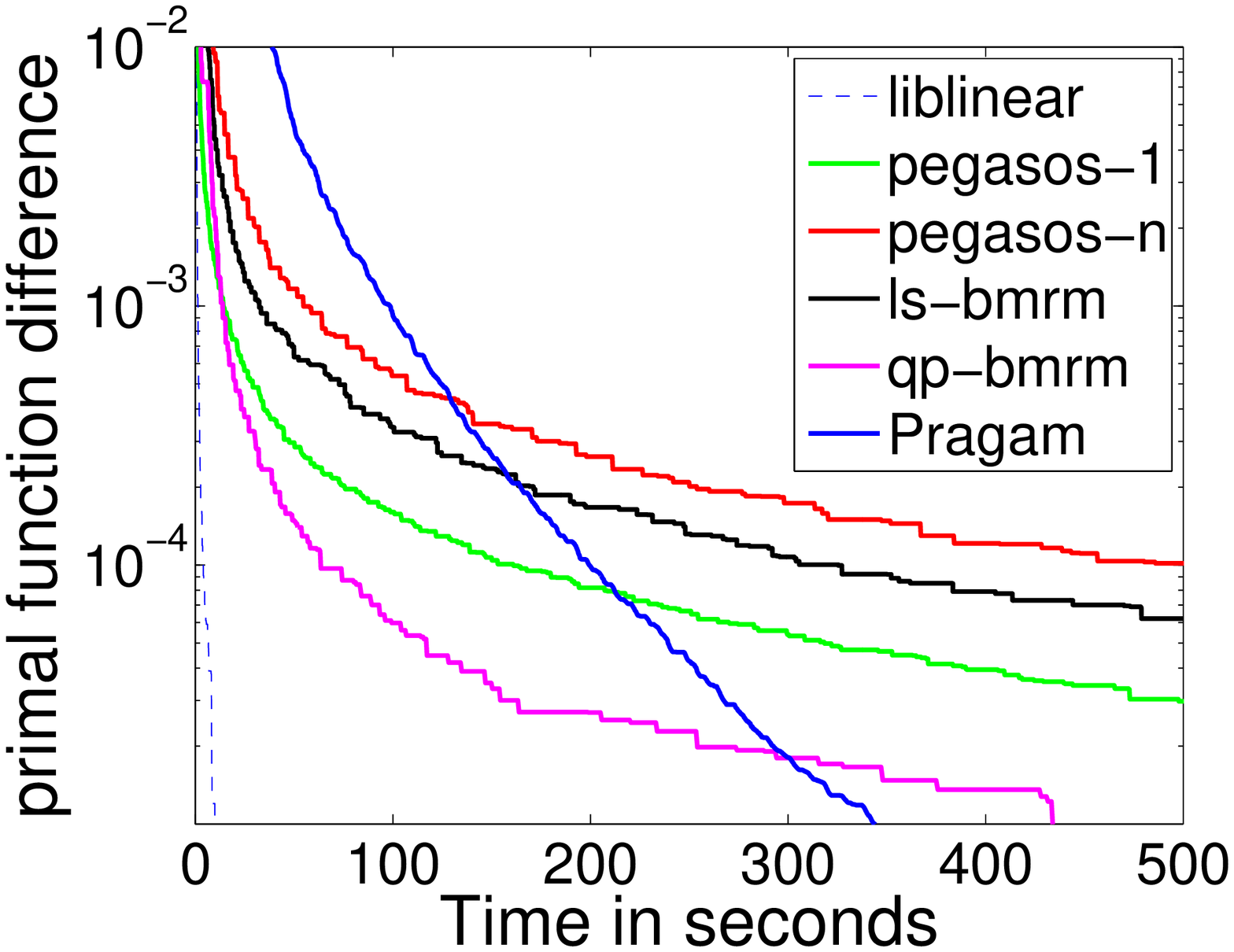}} ~~
\subfloat[web8]{
    \includegraphics[width=4.35cm]{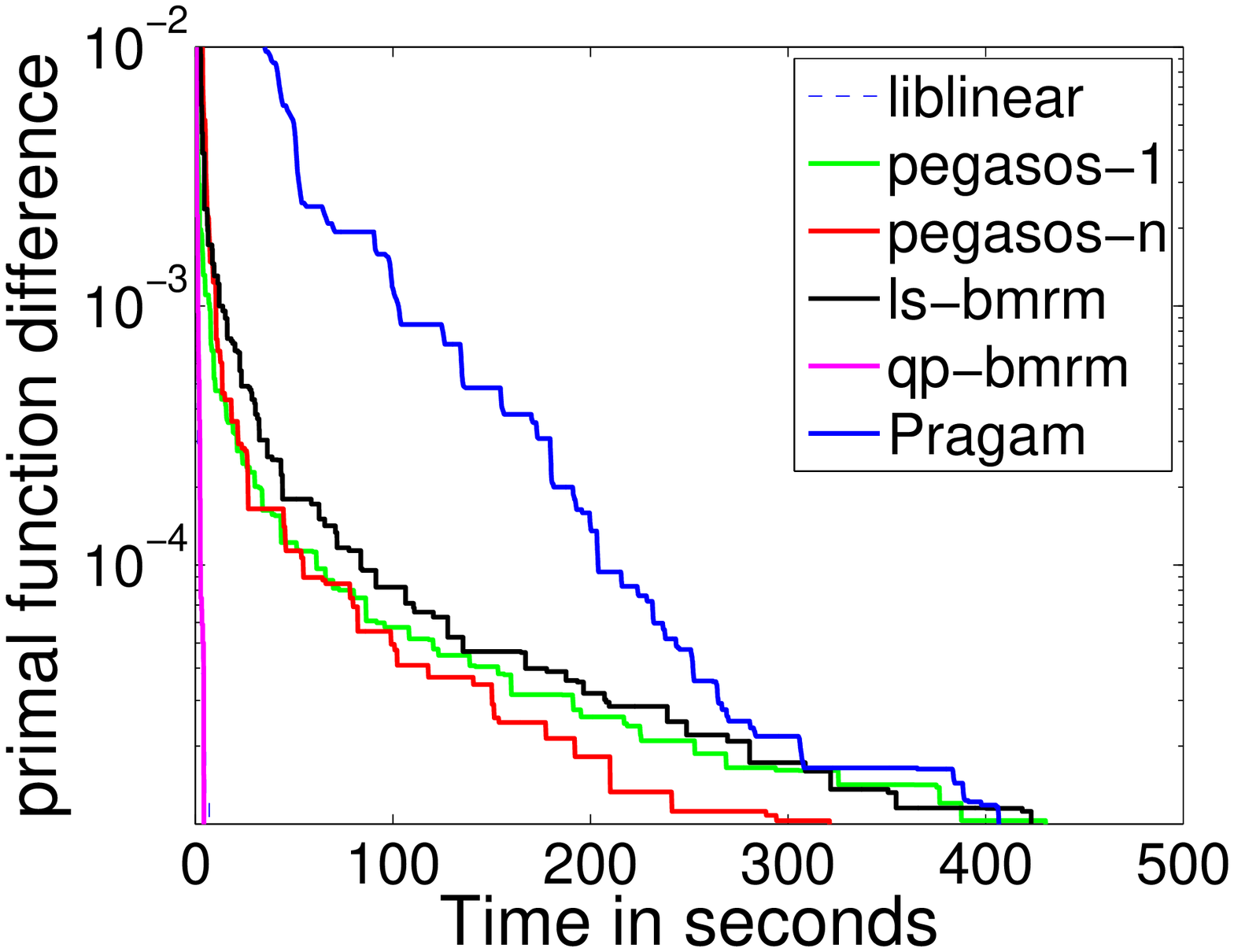}}
\caption{Primal function error versus time.}
\label{fig:fmin_diff_vs_time_app}
\end{centering}
\end{figure}

\begin{figure}[htbp]
\begin{centering}
\subfloat[adult9]{
    \includegraphics[width=4.35cm]{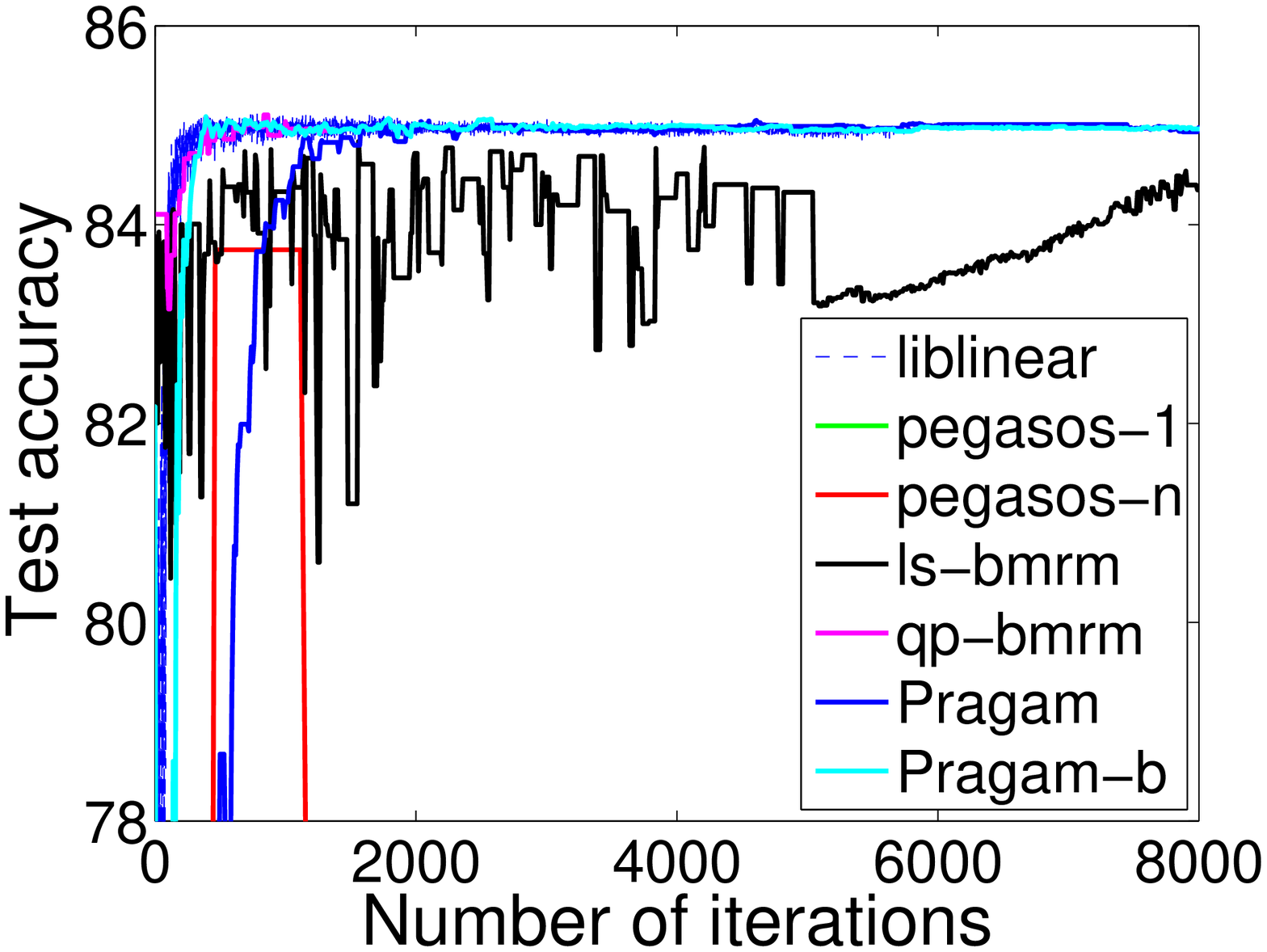}} ~~
\subfloat[astro-ph]{
    \includegraphics[width=4.35cm]{astro-ph_te_acc_vs_iter}} ~~
\subfloat[aut-avn]{
    \includegraphics[width=4.35cm]{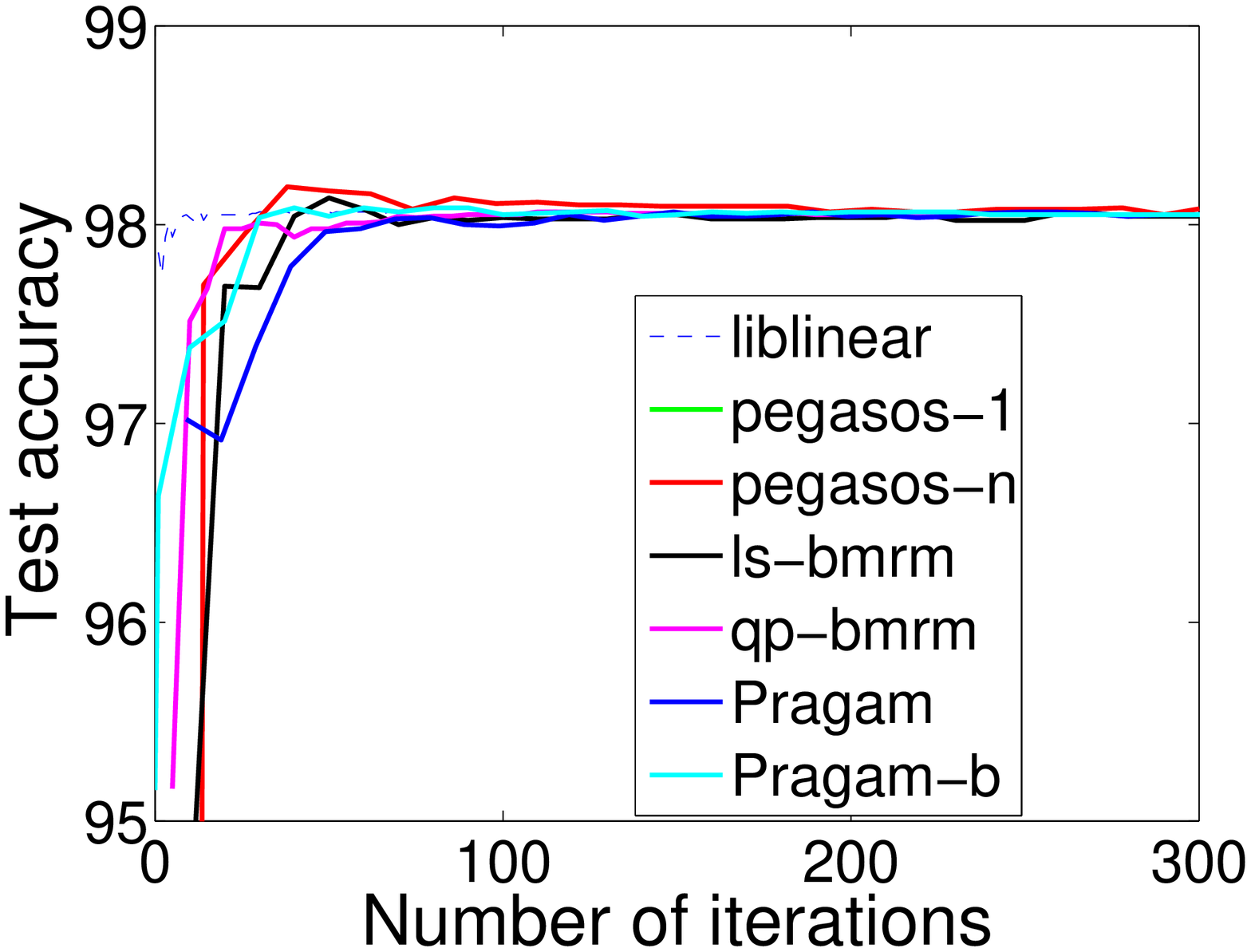}} \\
\subfloat[covertype]{
    \includegraphics[width=4.35cm]{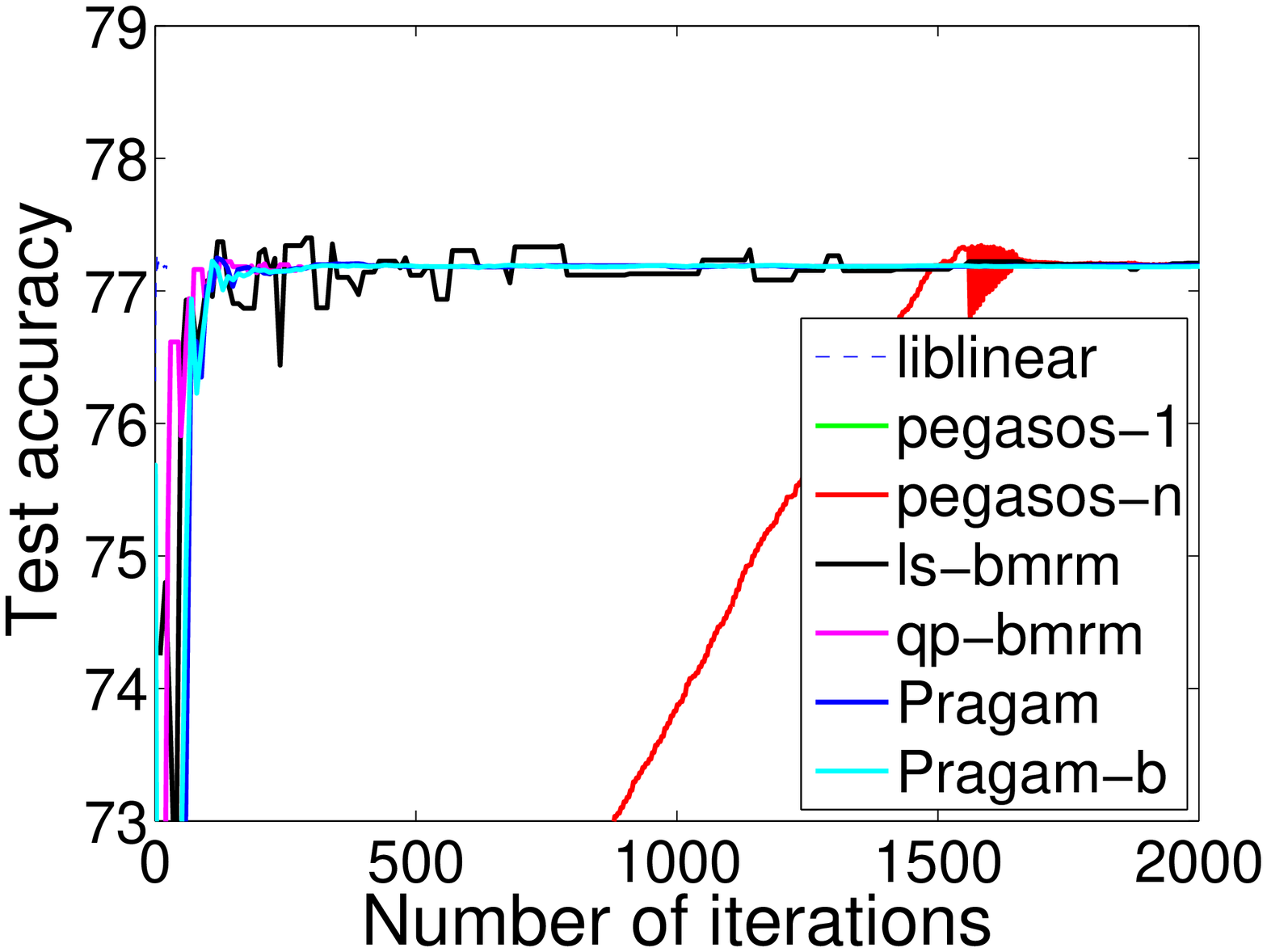}} ~~
\subfloat[news20]{
    \includegraphics[width=4.35cm]{news20_te_acc_vs_iter}} ~~
\subfloat[real-sim]{
    \includegraphics[width=4.35cm]{real-sim_te_acc_vs_iter}} \\
\subfloat[reuters-c11]{
    \includegraphics[width=4.35cm]{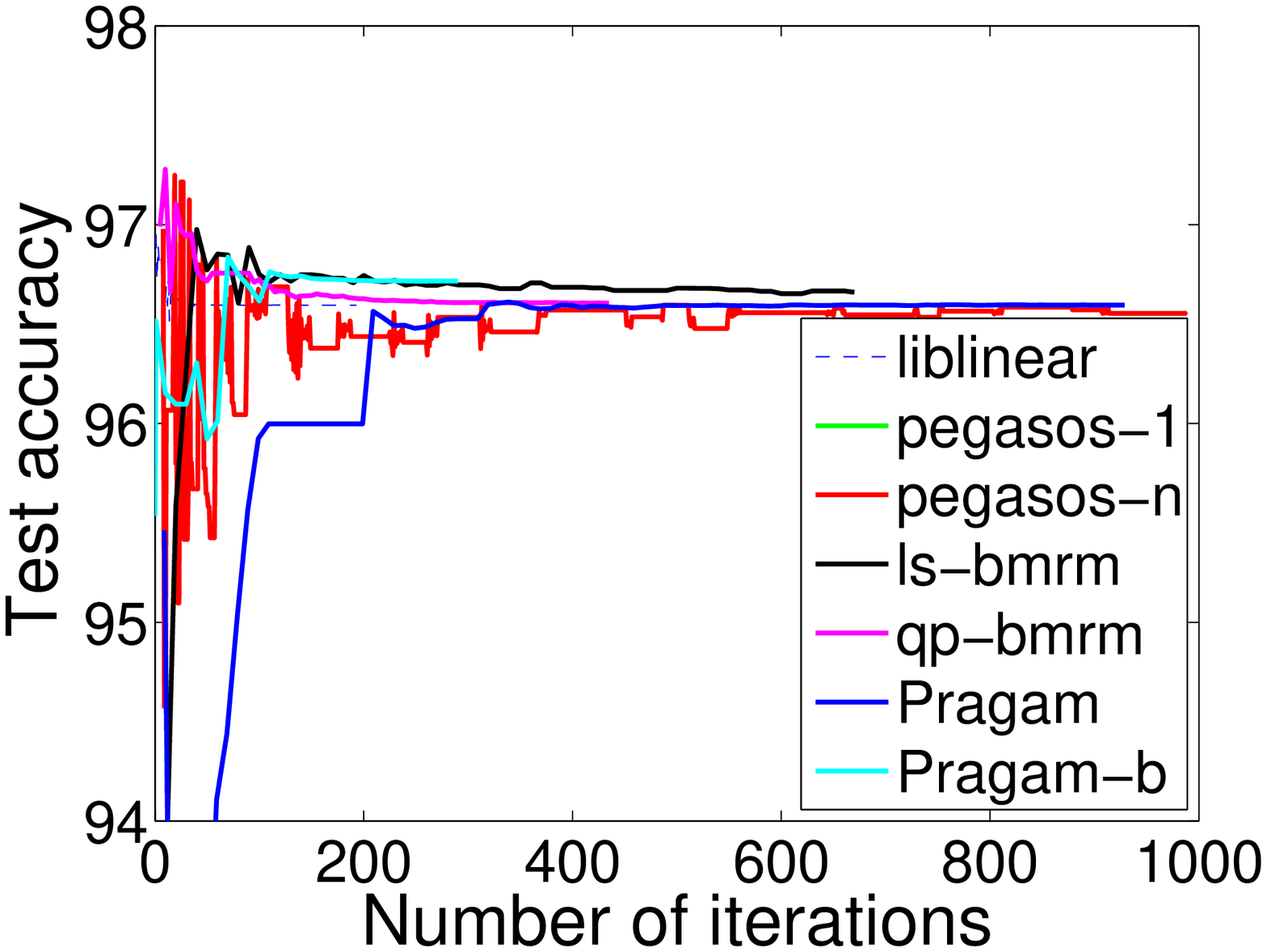}} ~~
\subfloat[reuters-ccat]{
    \includegraphics[width=4.35cm]{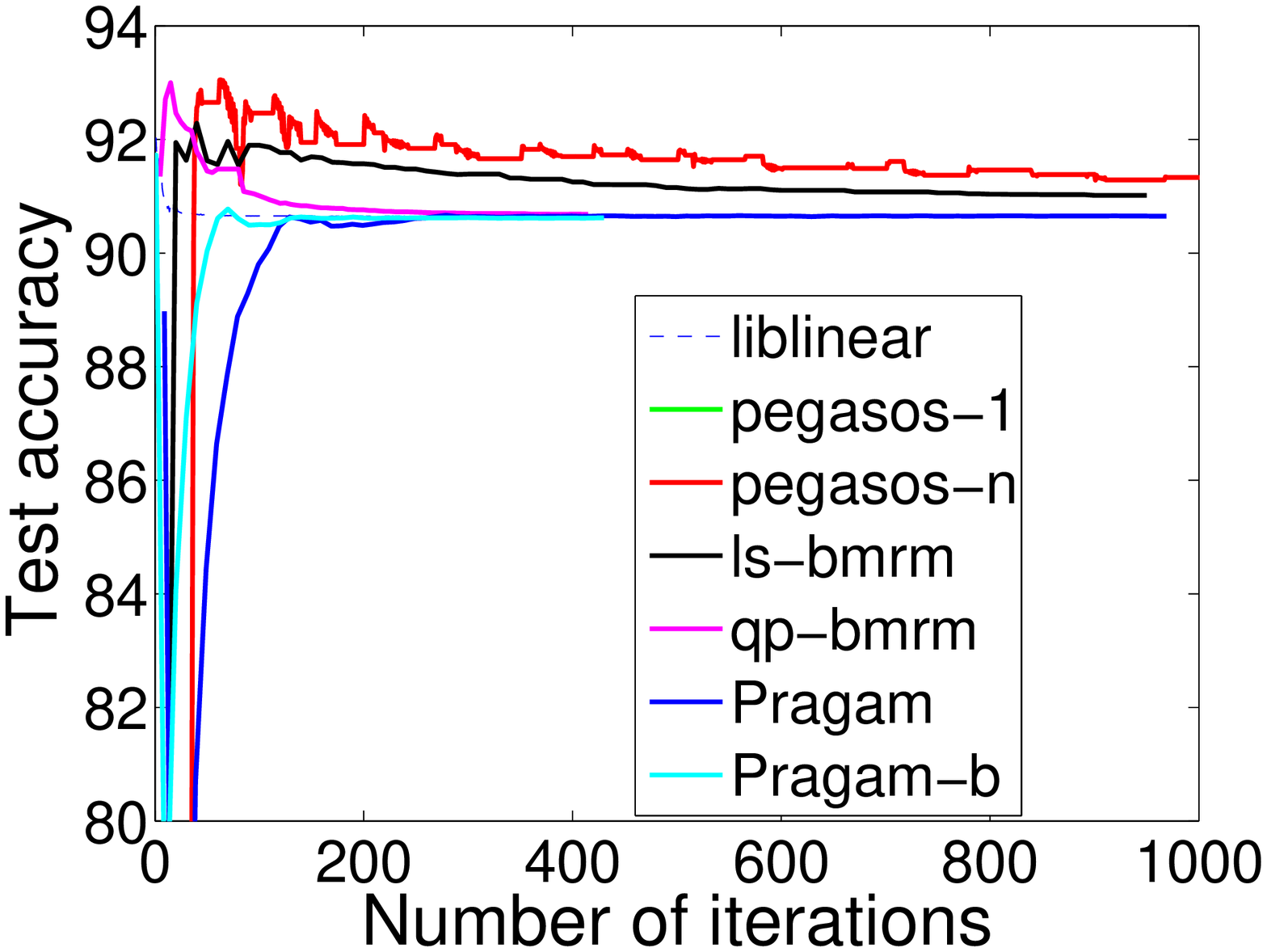}} ~~
\subfloat[web8]{
    \includegraphics[width=4.35cm]{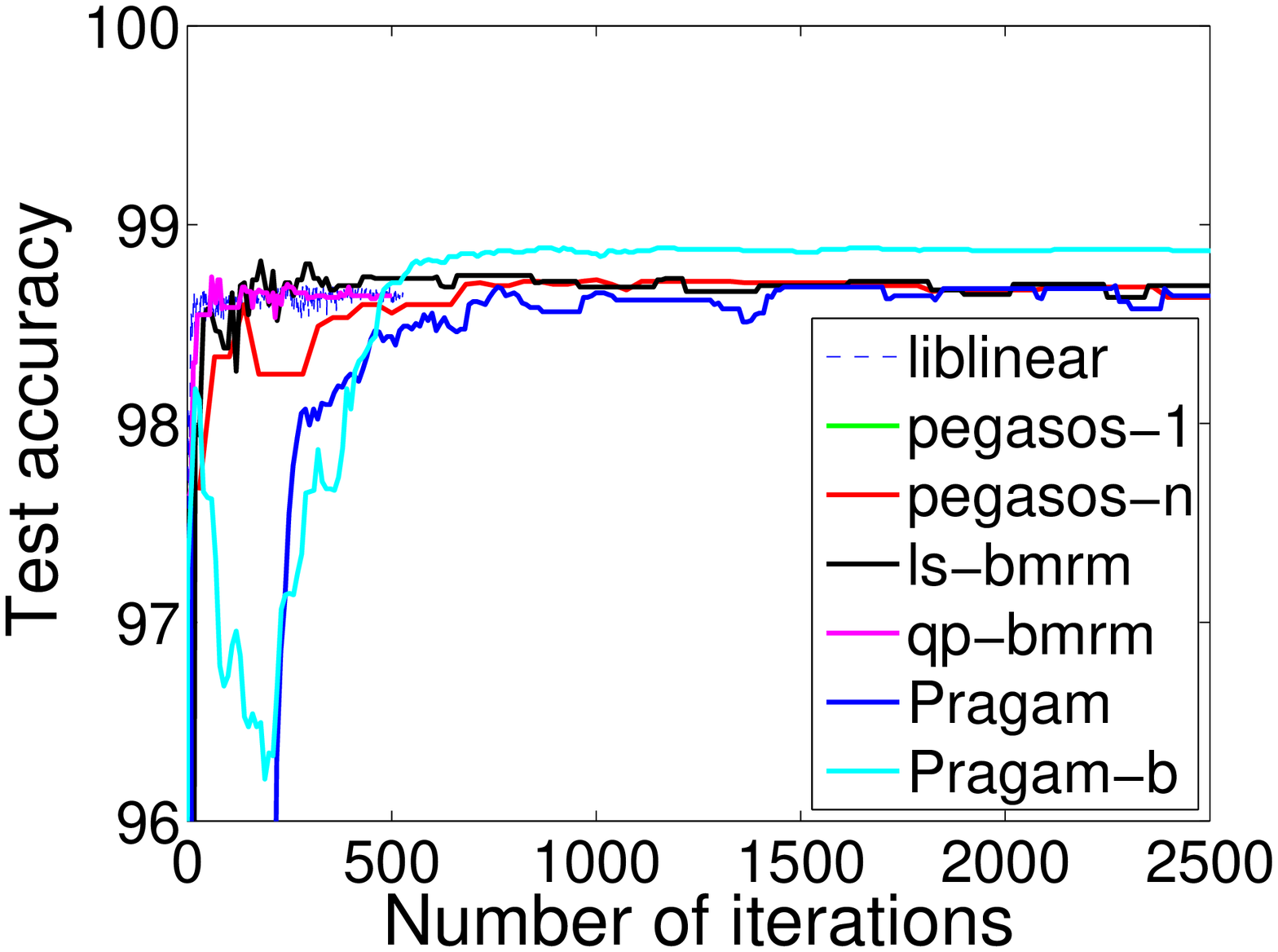}}
\caption{Test accuracy versus number of iterations.}
\label{fig:te_acc_vs_iter_app}
\end{centering}
\end{figure}

\begin{figure}[htbp]
\begin{centering}
\subfloat[adult9]{
    \includegraphics[width=4.35cm]{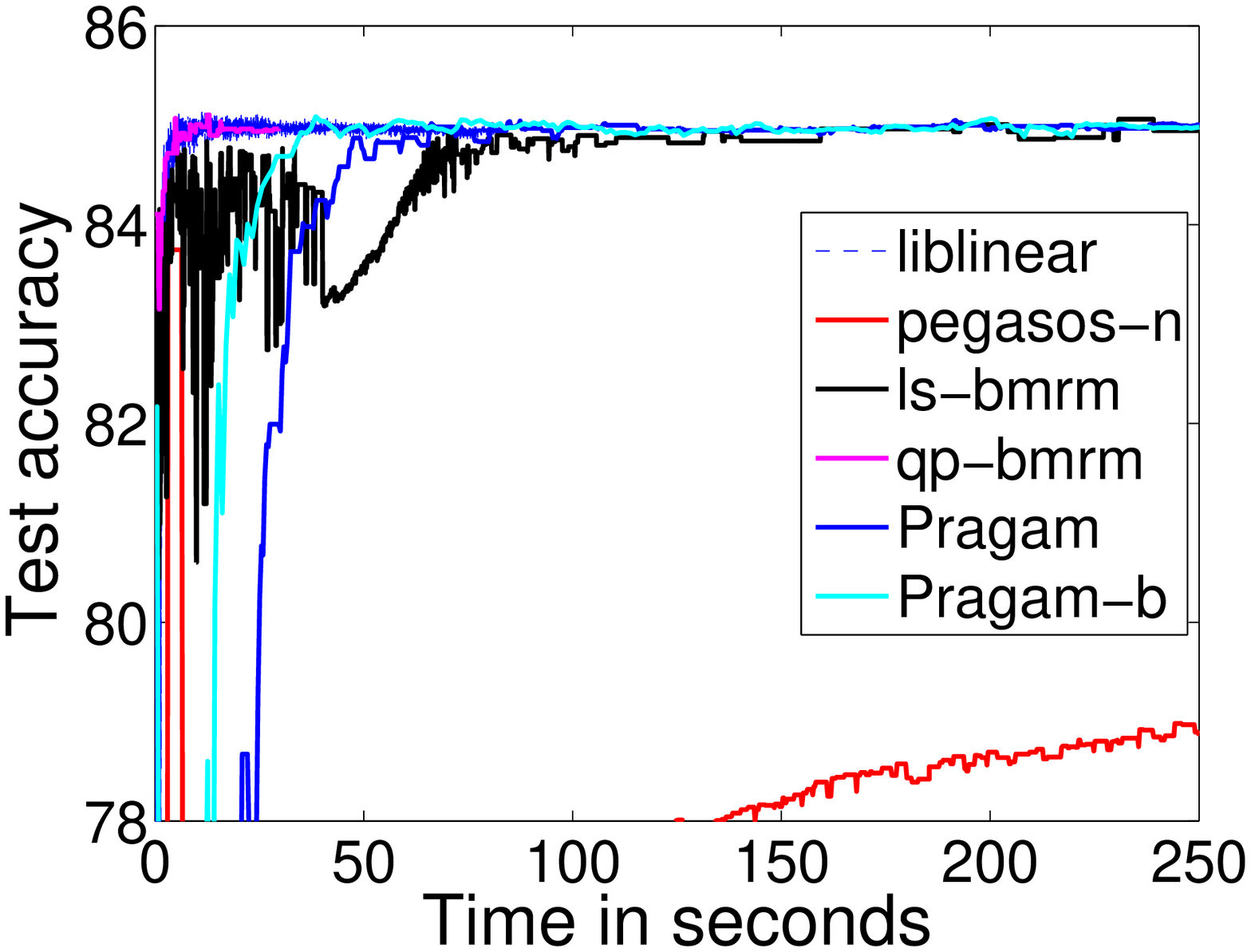}} ~~
\subfloat[astro-ph]{
    \includegraphics[width=4.35cm]{astro-ph_te_acc_vs_time}} ~~
\subfloat[aut-avn]{
    \includegraphics[width=4.35cm]{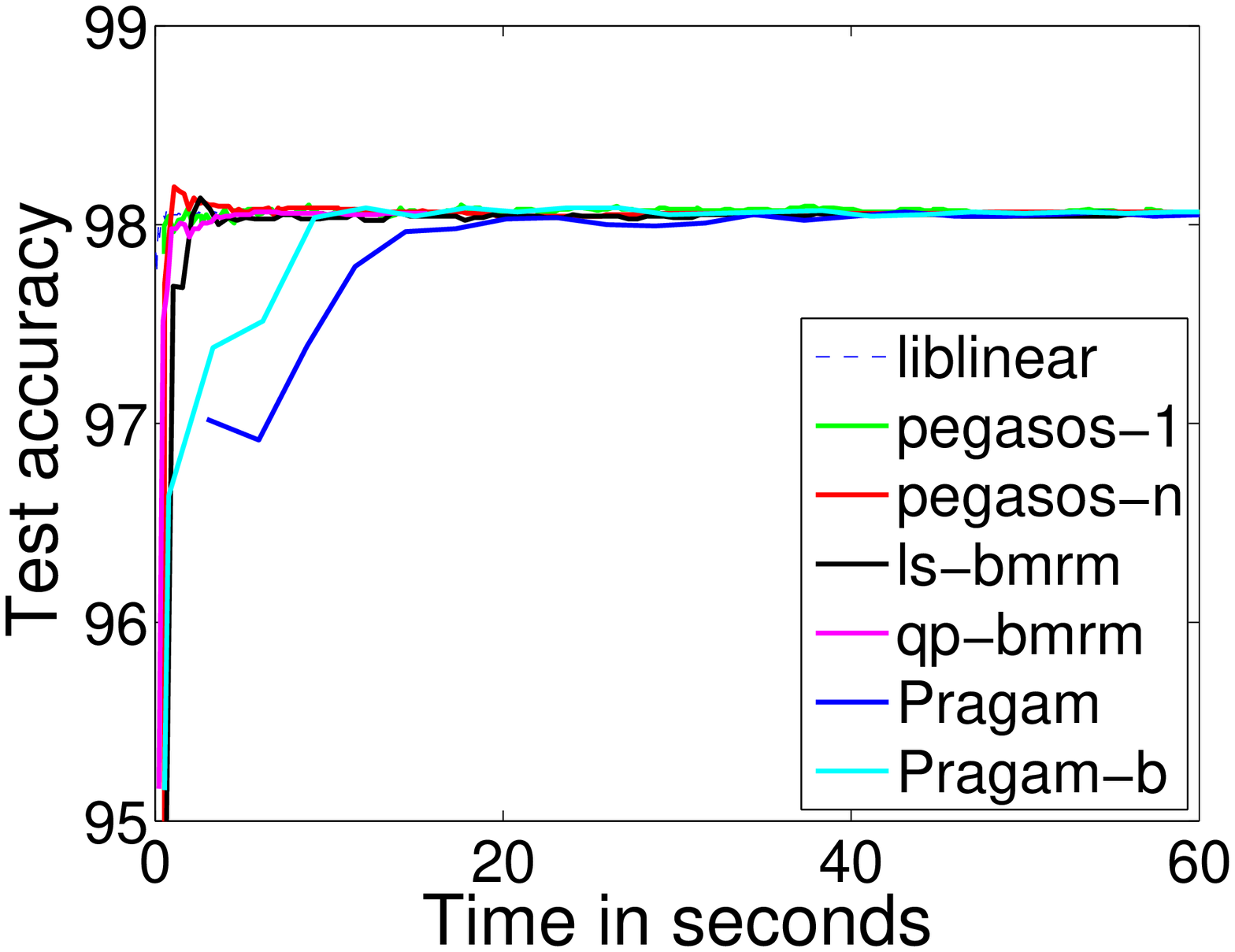}} \\
\subfloat[covertype]{
    \includegraphics[width=4.35cm]{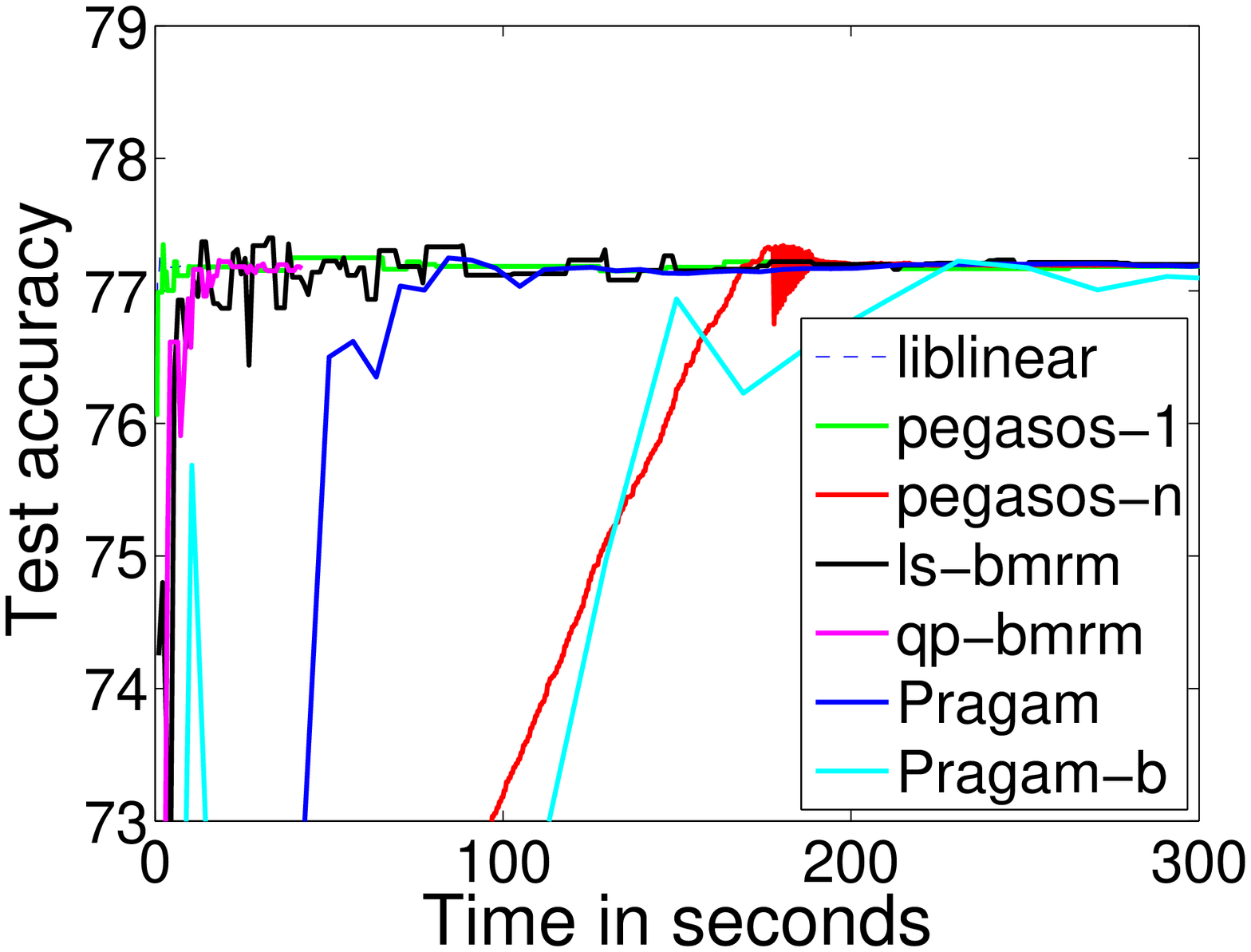}} ~~
\subfloat[news20]{
    \includegraphics[width=4.35cm]{news20_te_acc_vs_time}} ~~
\subfloat[real-sim]{
    \includegraphics[width=4.35cm]{real-sim_te_acc_vs_time}} \\
\subfloat[reuters-c11]{
    \includegraphics[width=4.35cm]{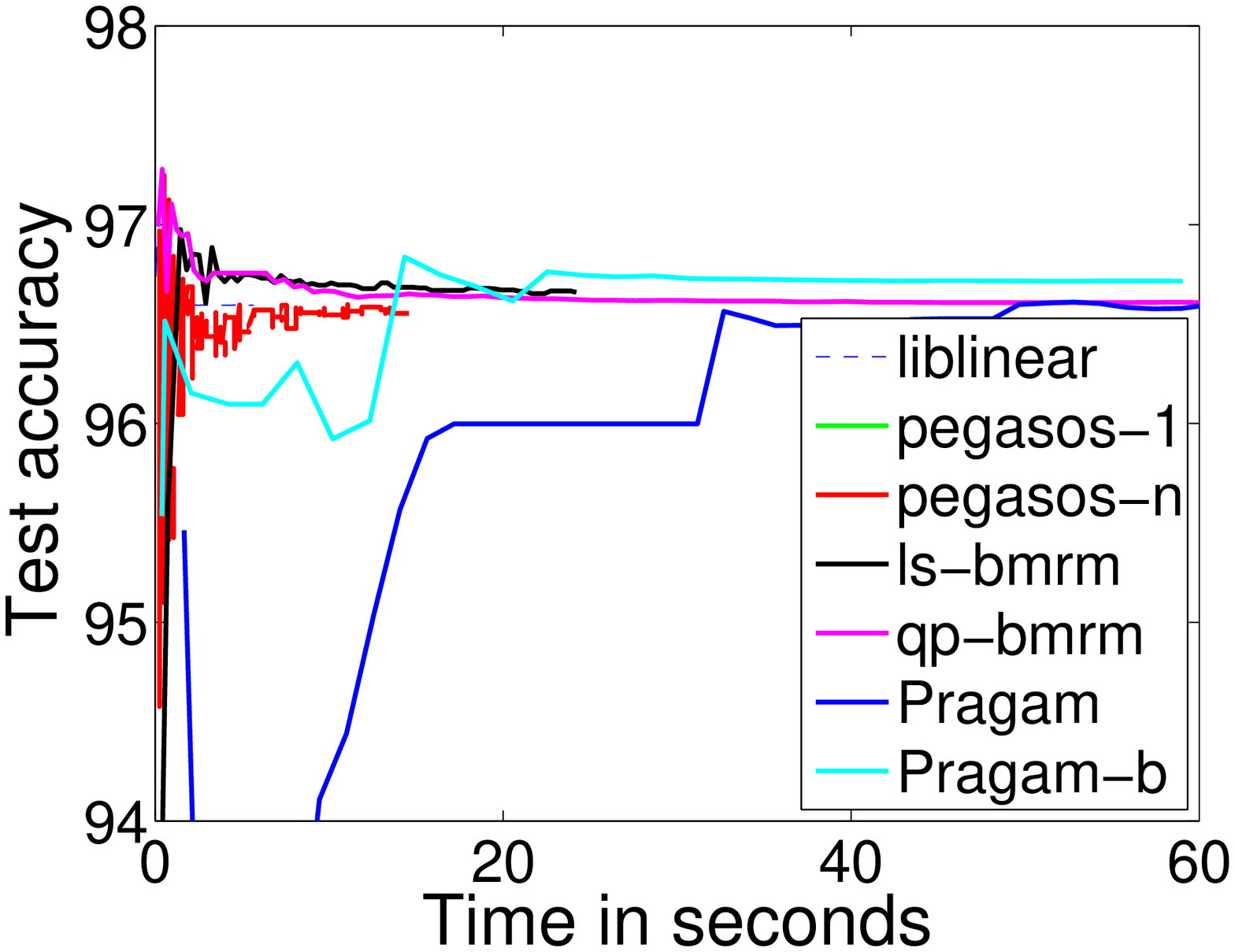}} ~~
\subfloat[reuters-ccat]{
    \includegraphics[width=4.35cm]{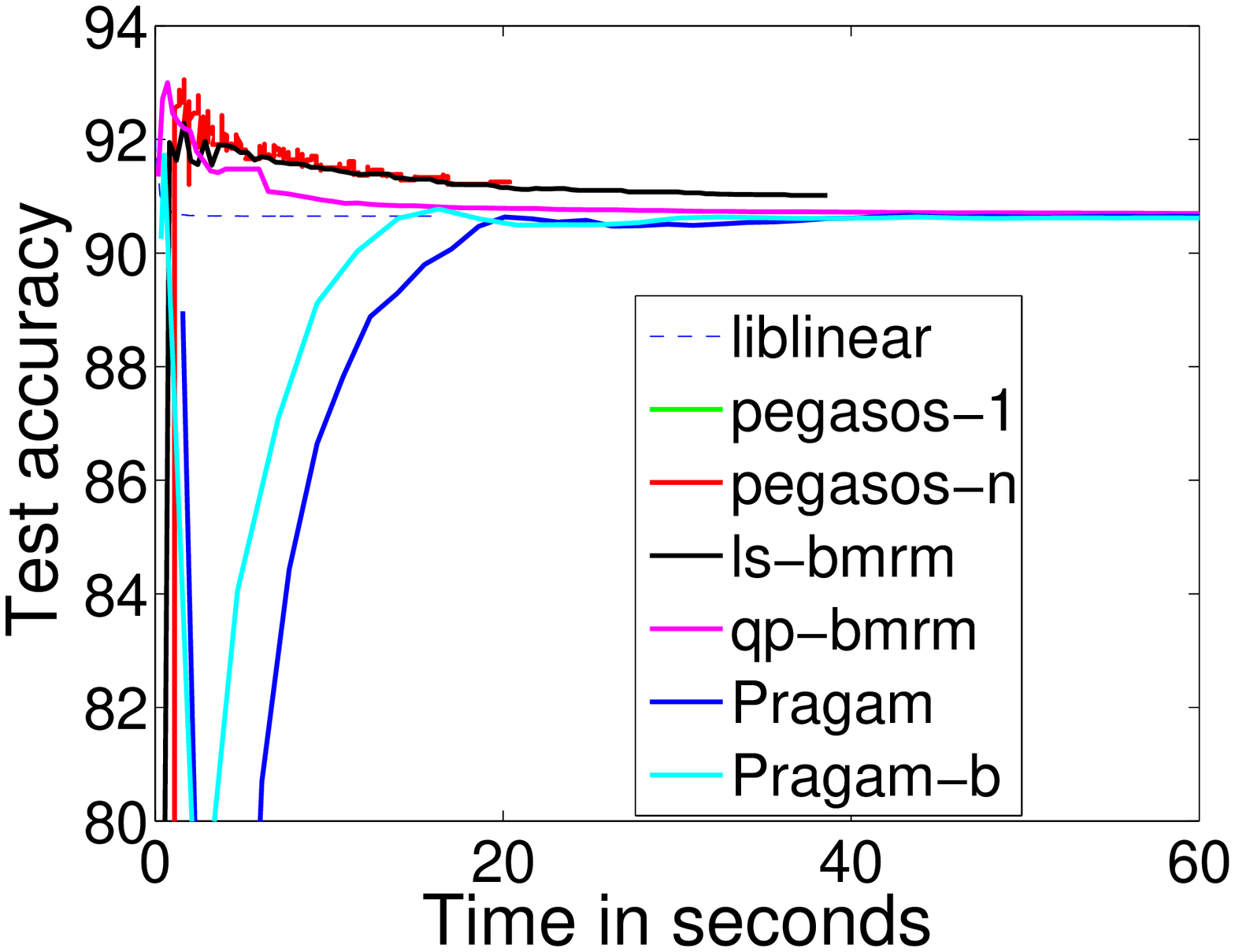}} ~~
\subfloat[web8]{
    \includegraphics[width=4.35cm]{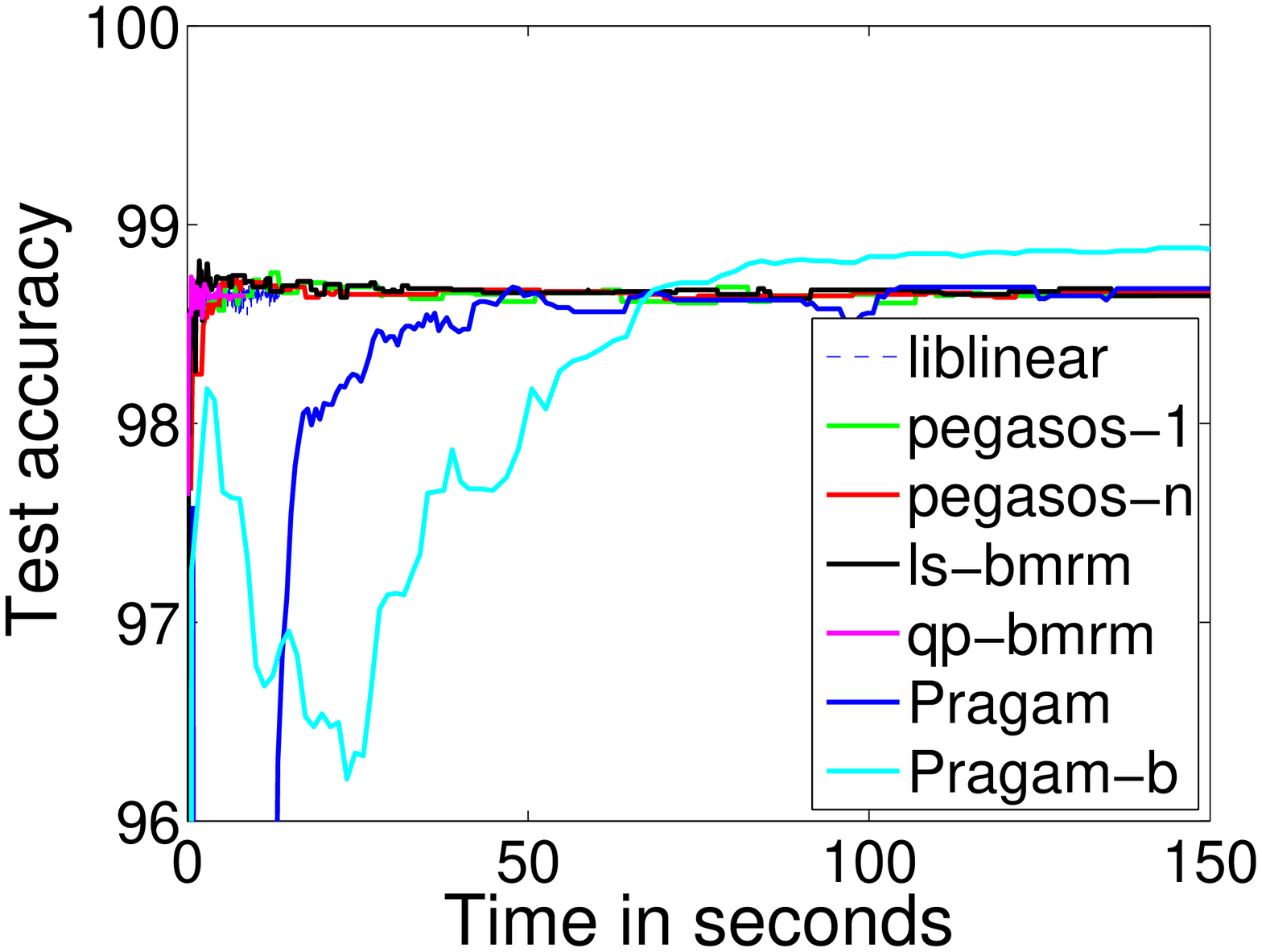}}
\caption{Test accuracy versus time.}
\label{fig:te_acc_vs_time_app}
\end{centering}
\end{figure}

\clearpage

\section{Structured Output}
\label{app:struct_output}

In this section, we show how \nest\ can be applied to structured output data. When the output space $\yb$ has structures, the label space becomes exponentially
large and the problem becomes much more expensive.  To make the computation tractable,
it is common to reparameterize on the cliques and to estimate the parameters via graphical
model inference.  We Below are two examples.

\subsection{Margin Scaled Maximum Margin Markov Network}
\label{sec:mcuben}

The maximum margin Markov network formulation (M$^3$N) by \cite{TasGueKol04} has the following
dual form (skipping the primal).  Here, a training example $i$ has a graphical model with
maximal clique set $\Ccal^{(i)}$.  For each $i$ and $c \in \Ccal^{(i)}$, $\alpha_{i,c}(\yb_c)$ stands for the marginal probability of the clique configuration $\yb_c$.  Any possible structured output $\yb$ can be measured against the given true label $\yb^i$ with a loss function $\ell_i(\yb)$.  It is assumed that $\ell_i(\yb)$ decomposes additively onto the cliques: $\ell_i(\yb) = \sum_{c \in \Ccal^{(i)}} \ell_{i,c}(\yb_c)$.  In addition, \cite{TasGueKol04} uses a joint feature map $\fb_i(\yb) := \fb(\xb^i, \yb)$, and define $\Delta \fb_i(\yb) := \fb_i(\yb^i) - \fb_i(\yb)$.  Again, we assume that $\fb_i(\yb)$ decomposes additively onto the cliques: $\fb_i(\yb) = \sum_c \fb_{i,c}(\yb_c)$, where $\fb_{i,c}(\yb_c) := \fb_c(\xb^i, \yb_c)$.  The simplest joint feature map is defined by:
\begin{align*}
\inner{\fb_c(\xb, \yb_c)}{\fb_{c'}(\xbbar, \ybbar_{c'})} = \delta(\yb_c = \ybbar_{c'}) k(\xb_c, \xbbar_{c'}) = \delta(\yb_c = \ybbar_{c'}) \inner{\xb_c}{\xbbar_{c'}}.
\end{align*}
Notice that $c$ and $c'$ are not required to be of the same ``type", and $\delta(\yb_c = \ybbar_{c'})$ will automatically filter out incompatible types, \eg, $c$ is an edge and $c'$ is a node.  This kernel can be easily vectorized into $\fb(\xb, \yb) := \sum_c \left(\xb_c \bigotimes (\delta(\yb_c = \yb_{c,1}), \ldots, \delta(\yb_c = \yb_{c,m(c)}))^{\top}\right)$, where $m(c)$ is the number of label configurations that clique $c$ can take, and $\bigotimes$ (cross product) is defined by:
\begin{align*}
  \otimes: \RR^s \times \RR^t \to \RR^{st}, \qquad (\ab \otimes \bb)_{(i-1)t + j} := a_i b_j.
\end{align*}

\begin{align*}
  \min_{\alphab} \qquad \frac{C}{2} \nbr{\sum_{i, c, \yb_c} \alpha_{i,c}(\yb_c) \Delta \fb_{i,c} (\yb_c)}^2 &- \sum_{i, c, \yb_c} \alpha_{i,c}(\yb_c) \ell_{i,c}(\yb_c) \\
s.t. \qquad \qquad \qquad \quad \sum_{\yb_c} \alpha_{i,c}(\yb_c) &= 1 \qquad \qquad \qquad \qquad \quad \forall i, \forall c \in \Ccal^{(i)}; \\
\alpha_{i,c}(\yb_c) &\ge 0 \qquad \qquad \qquad \qquad \quad \forall i, \forall c \in \Ccal^{(i)}, \forall \yb_c; \\
\sum_{\yb_c \sim \yb_{c \cap c'}} \alpha_{i,c}(\yb_c) &= \sum_{\yb_{c'} \sim \yb_{c \cap c'}} \alpha_{i,c'}(\yb_{c'}) \qquad \forall i, \forall c, c' \in \Ccal^{(i)}: c \cap c' \neq \emptyset, \forall \yb_{c \cap c'}.
\end{align*}

The subroutines of mapping in \eqref{eq:alpha_mu} and \eqref{eq:v_alpha} can be viewed as projecting a vector onto the probability simplex under $L_2$ distance.  Moreover, the image now is restricted to the pmf which satisfies the conditional independence properties encoded by the graphical model.  This is much more difficult than projecting onto a box with a linear equality constraint as in SVM, and we can only resort to a block coordinate descent as detailed in Appendix \ref{sec:app_proj_graph}.

\subsection{Gaussian Process Sequence Labeling}
\cite{AltHofSmo04} proposed using Gaussian process to segment and annotate sequences.  It assumes that all training sequences have the same length $l$, and each node can take values in $[m]$. For sequences, the maximum cliques are edges: $\Ccal^{(i)} = \cbr{(t, t+1) : t \in [l - 1]}$.  We use $\alphab$ to stack up the marginal distributions on the cliques:
\[
\cbr{\alpha_{i,c}(\yb_c) : i, c \in \Ccal^{(i)}, \yb_c} = \cbr{\alpha_{i,t}(y_t, y_{t+1}): i, t \in [l-1], y_t, y_{t+1} \in [m]}.
\]

The marginal probability $\alpha_{i,c}(\yb_c)$ just aggregates relevant elements in the joint distribution via a linear transform.  With the joint distribution vector $p \in \RR^{n m^l}$, we can write $\alphab = \Lambda p$, where $\Lambda$ is $n l m^2 \times n m^l$ defined by
\[\lambda_{(j, t, \sigma, \tau), (i, \yb)} = \delta (i = j \wedge y_t = \sigma \wedge y_{t+1} = \tau).\]
A key difference from the M$^3$N is that in Gaussian process, the constraint that the joint density lying in the probability simplex is replaced by regularizing the log partition function.  So the set of marginal distributions on cliques do not have local consistency constraints, \ie, they are free variables.  The ultimate optimization problem in \cite{AltHofSmo04} is an unconstrained optimization:
\begin{align}
\label{eq:gp_seq_dual}
  \min_{\alphab} \qquad \alphab^{\top} K \alphab - \sum_{i=1}^n \alphab^{\top} K \Lambda e_{(e,\yb_i)} + \sum_{i=1}^n \log \sum_{\yb} \exp \left( \alphab^{\top} K \Lambda e_{i,\yb} \right).
\end{align}
where $K$ is a kernel on $\cbr{(\xb^i, (y^i_t, y^i_{t+1})) : i, t \in [l-1], y_t, y_{t+1} \in [m]}$.  A simple example is:
\begin{align*}
k((\xb^i, (y^i_t, y^i_{t+1})), (\xb^j, (y^j_{t'}, y^j_{t'+1})) = &\delta(t = t') \left( \delta (y^i_t = y^j_{t'} \wedge y^i_{t+1} = y^j_{t'+1}) +\right.\\
&\left. k(\xb^i_t, \xb^j_{t'}) \delta(y^i_t = y^j_{t'}) +  k(\xb^i_{t+1}, \xb^j_{t'+1}) \delta(y^i_{t+1} = y^j_{t'+1}) \right).
\end{align*}
If stationality is assumed, we can drop the $\delta(t = t')$ above and allow node swapping:
\begin{align*}
k((\xb^i, (y^i_t, y^i_{t+1})), (\xb^j, (y^j_{t'}, y^j_{t'+1})) = &\delta (y^i_t = y^j_{t'} \wedge y^i_{t+1} = y^j_{t'+1}) \\
&+ \sum_{p \in \cbr{t, t+1}} \sum_{q \in \cbr{t', t'+1}} k(\xb^i_p, \xb^j_q) \delta(y^i_p = y^j_q).
\end{align*}

The gradient of the first two terms in \eqref{eq:gp_seq_dual} can be computed with ease.  The gradient of the last term is
\begin{align*}
  \grad_{\alphab} \sbr{\log \sum_{\yb} \exp \left( \alphab^{\top} K \Lambda e_{i,\yb} \right)} = K \EE_{Y} \sbr{\Lambda e_{i,Y}}.
\end{align*}
\cite{AltHofSmo04} computes this expectation via the forward-backward algorithm.  The projections \eqref{eq:alpha_mu} and \eqref{eq:v_alpha} are the same as in Appendix \ref{sec:mcuben}, and Appendix \ref{sec:app_proj_seq} provides a solver for the special case of sequence.

\subsection{Efficient projection onto $n$ dimensional simplex factorized by a graphical model}
\label{sec:app_proj_graph}

As a simple extension of Appendix \ref{sec:simple_qp}, we now consider a more involved case.  In addition to projecting onto the $n$ dimensional simplex under $L_2$ distance, we also restrict that the image is factorized according to a graphical model.  Formally, suppose we have a graphical model over $n$ random variables with maximum clique set $\Ccal$.  For each clique $c$, suppose the set of all its possible configuration is $V_c$, and the pmf of the marginal distribution on clique $c$ is $\alpha_c (\yb_c)$, where $\yb_c \in V_c$.  Given a set $\cbr{\mb_c \in \RR^{\abr{V_c}} : c \in \Ccal }$, we want to find a set $\cbr{\alphab_c \in \RR^{\abr{V_c}} : c \in \Ccal }$ which minimizes:
\begin{align*}
  \min \qquad \frac{1}{2} \sum_c d_c^2 &\nbr{\alphab_c - \mb_c}_2^2 \\
  s.t. \qquad \sum_{\yb_c} \alpha_c(\yb_c) &= 1 \qquad \qquad \forall c \in \Ccal; \\
  \alpha_c(\yb_c) &\ge 0 \qquad \qquad \forall c \in \Ccal, \forall \yb_c; \\
  \sum_{\yb_c \sim \yb_{c \cap c'}} \alpha_c(\yb_c) &= \sum_{\yb_{c'} \sim \yb_{c \cap c'}} \alpha_{c'} (\yb_{c'}) \quad \forall c \cap c' \neq \emptyset, \forall \yb_{c \cap c'}.
\end{align*}
The last (set of) constraint enforces the local consistency of the marginal distributions, and $\yb_c \sim \yb_{c \cap c'}$ means the assignment of clique $c$ matches $\yb_{c \cap c'}$ on the subset $c \cap c'$.  If the graphical model is tree structured, then it will also guarantee global consistency.  We proceed by writing out the standard Lagrangian:
\begin{align*}
  L = &\frac{1}{2} \sum_c d_c^2 \sum_{\yb_c} (\alpha_c(\yb_c) - m_c(\yb_c))^2 - \sum_c \lambda_c \left( \sum_{\yb_c} \alpha_c(\yb_c) - 1 \right) - \sum_{c, \yb_c} \xi_c(\yb_c) \alpha_c (\yb_c) \\
  &- \sum_{c, c' : c \cap c' \neq \emptyset} \sum_{\yb_{c \cap c'}} \mubar_{c,c'}(\yb_{c \cap c'}) \left(\sum_{\yb_c : \yb_c \sim \yb_{c \cap c'}} \alpha_c(\yb_c) - \sum_{\yb_{c'} \sim \yb_{c \cap c'}} \alpha_{c'}(\yb_{c'}) \right).
\end{align*}
Taking derivative over $\alpha_c(y_c)$:
\begin{align}
  \frac{\partial L}{\partial \alpha_c(\yb_c)} &= d_c^2 (\alpha_c (\yb_c) - m_c(\yb_c)) - \lambda_c - \xi_c(\yb_c) - \sum_{c'} \mubar_{c,c'}(\yb_{c \cap c'}) + \sum_{c'} \mubar_{c',c}(\yb_{c \cap c'}) = 0, \nonumber \\
\label{eq:dual_connect_struct_simplex}
  &\Rightarrow \alpha_c(\yb_c) = m_c(\yb_c) + d_c^{-2} \left(\lambda_c + \xi_c(\yb_c) + \sum_{c'} \mu_{c,c'} (\yb_{c \cap c'}) \right),
\end{align}
where $\mu_{c,c'} (\yb_{c \cap c'}) := \mubar_{c,c'}(\yb_{c \cap c'}) - \mubar_{c',c}(\yb_{c \cap c'})$.  Plugging it back into $L$, we derive the dual problem:
\begin{align}
\label{eq:dual_struct_simplex_pro}
  \min D(\lambda_c, \xi_c(\yb_c), \mu_{c,c'} (\yb_{c \cap c'})) = &\frac{1}{2} \sum_c d_c^{-2} \sum_{\yb_c} \left( \lambda_c + \xi_c(\yb_c) + \sum_{c'} \mu_{c,c'} (\yb_{c \cap c'}) \right)^2 \\
  &+ \sum_c \sum_{\yb_c} m_c(\yb_c) \left(\lambda_c + \xi_c(\yb_c) + \sum_{c'} \mu_{c,c'} (\yb_{c \cap c'}) \right) - \sum_c \lambda_c \nonumber \\
  s.t. \qquad \xi_c (\yb_c) \ge 0. \nonumber
\end{align}
Looking at the problem, it is essentially a QP over $\lambda_c, \xi_c(\yb_c), \mu_{c,c'} (\yb_{c \cap c'})$ with the only constraint that $\xi_c (\yb_c) \ge 0$.  Similar to \ref{sec:simple_qp}, one can write $\xi_c (\yb_c)$ as a hinge function of $\lambda_c$ and $\mu_{c,c'} (\yb_{c \cap c'})$.  However since it is no longer a single variable function, it is very hard to apply the median trick here.  So we resort to a simple block coordinate descent.  The optimization steps are given in Algorithm \ref{algo:structured_simplex_projection} with reference to the following expressions of gradient:
\begin{subequations}
\begin{align}
\label{eq:parD_parXi}
  \frac{\partial D}{\partial \xi_c(\yb_c)} &= -d_c^{-2} (\lambda_c + \xi_c(\yb_c) + \sum_{c'} \mu_{c,c'}(\yb_{c \cap c'})) + m_c(\yb_c) = 0  \\
\label{eq:parD_parLambda}
  \frac{\partial D}{\partial \lambda_c} &= d_c^{-2} \sum_{\yb_c} \left( \lambda_c + \xi_c(\yb_c) + \sum_{c'} \mu_{c,c'} (\yb_{c \cap c'})  \right) + \sum_{\yb_c} m_c(\yb_c) - 1 \\
\label{eq:parD_parMu}
  \frac{\partial D}{\partial \mu_{c,c'} (\yb_{c \cap c'})} &= d_c^{-2} \sum_{\yb'_c \sim \yb_{c \cap c'}} \left( \lambda_c + \xi_c(\yb'_c) + \sum_{\cbar} \mu_{c,\cbar} (\yb'_{c,\cbar}) \right) + \sum_{\yb'_c \sim \yb_{c \cap c'}} m_c(\yb'_c).
\end{align}
\end{subequations}
From \eqref{eq:parD_parXi} and $\xi_c(\yb_c) \ge 0$, we can derive
\begin{align}
\label{eq:xi_as_hinge}
  \xi_c(\yb_c) = \sbr{-d_c^2 m_c(\yb_c) - \lambda_c - \sum_{c'} \mu_{c,c'} (y_{c \cap c'})}_+.
\end{align}

\begin{algorithm}[t]
\begin{algorithmic}[1]
    \caption{\label{algo:structured_simplex_projection} A coordinate descent scheme for minimizing the dual problem \eqref{eq:dual_struct_simplex_pro}.}
    \STATE Randomly initialize $\cbr{\lambda_c : c}, \cbr{\xi_c(\yb_c):c,\yb_c}, \cbr{\mu_{c,c'} (\yb_{c \cap c'}) : c, c', \yb_{c \cap c'}}$.
    \WHILE{not yet converged}
        \STATE Fixing $\xi_c(\yb_c)$, apply conjugate gradient to minimize the unconstrained quadratic form in \eqref{eq:dual_struct_simplex_pro} with respect to $\cbr{\lambda_c : c}$ and $\cbr{\mu_{c,c'} (\yb_{c \cap c'}) : c, c', \yb_{c \cap c'}}$.  The necessary gradients are given in \eqref{eq:parD_parLambda} and \eqref{eq:parD_parMu}.
        \STATE Set $\xi_c(\yb_c) \leftarrow \sbr{-d_c^2 m_c(\yb_c) - \lambda_c - \sum_{c'} \mu_{c,c'} (y_{c \cap c'})}_+$ for all $c \in \Ccal$ and $\yb_c$.
    \ENDWHILE
    \STATE Compute $\alpha_c(\yb_c)$ according to \eqref{eq:dual_connect_struct_simplex}.
\end{algorithmic}
\end{algorithm}

\subsubsection{Special case: sequence}
\label{sec:app_proj_seq}

Suppose the graph is simply a sequence: $x_1 - x_2 - \ldots - x_L$ and each node can take value in $[m]$.  Then the cliques are $\cbr{(x_t, x_{t+1}) : t \in [L-1]}$ and the primal is:
\begin{align*}
  \min \qquad \frac{1}{2} \sum_{t=1}^{L-1} d_t^2 \sum_{i,j=1}^m &(\alpha_t(i,j) - m_t(i,j))^2 && \\
  s.t. \qquad \sum_{i,j} \alpha_t(i,j) &= 1 \qquad  &&\forall t \in [L-1]; \\
  \alpha_t(i, j) &\ge 0 \qquad  &&\forall t \in [L-1], i, j \in [m]; \\
  \sum_i \alpha_t (i, j) &= \sum_k \alpha_{t+1} (j, k)  \qquad &&\forall t \in [L-2], j \in [m].
\end{align*}
Proceeding with the standard Lagrangian:
\begin{align*}
  L = &\sum_{t=1}^{L-1} d_t^2 \sum_{i,j=1}^m (\alpha_t(i,j) - m_t(i,j))^2 - \sum_{t=1}^{L-1} \lambda_t \left( \sum_{i,j} \alpha_t(i,j) - 1 \right) - \sum_{t=1}^{L-1} \sum_{i,j} \xi_t(i,j) \alpha_t(i,j) \\
  &-\sum_{t=1}^{L-2} \sum_j \mu_t(j)\left( \sum_i \alpha_t(i,j) - \sum_k \alpha_{t+1}(j,k) \right).
\end{align*}
Taking derivative over $\alpha_t(i,j)$:
\begin{align}
  \frac{\partial L}{\partial \alpha_t(i,j)} &= d_t^2(\alpha_t(i,j) - m_t(i,j)) - \lambda_t - \xi_t(i,j) - \mu_t(j) + \mu_{t-1}(i) = 0 \nonumber \\
\label{eq:dual_connect_seq}
  &\Rightarrow \alpha_t(i,j) = d_t^{-2} (\lambda_t + \xi_t(i,j) + \mu_t(j) - \mu_{t-1}(i)) + m_t(i,j),
\end{align}
where we define $\mu_0(j) := 0$.  Plugging into $L$, we derive the dual problem:
\begin{align}
  \min D(\lambda_t, \xi_t(i,j), \mu_t(i)) = &\frac{1}{2} \sum_{t=1}^{L-1} d_t^2 \sum_{i,j} (\lambda_t + \xi_t(i,j) + \mu_t(j) - \mu_{t-1}(i))^2  \\
  &+ \sum_{t=1}^{L-1} \sum_{i,j} m_t(i,j) (\lambda_t + \xi_t(i,j) + \mu_t(j) - \mu_{t-1}(i)) - \sum_{t=1}^{L-1} \lambda_t \nonumber \\
  s.t. \qquad \xi_t(i,j) &\ge 0. \quad \forall t \in [L-1], \, i, j \in [m]. \nonumber
\end{align}
\begin{align*}
  \frac{\partial D}{\partial \xi_t(i,j)} &= d_t^{-2} (\lambda_t + \xi_t(i,j) + \mu_t(j) - \mu_{t-1}(i)) + m_t(i,j) = 0  \qquad \qquad \qquad \quad \forall t \in [L-1]\\
  &\Rightarrow \xi_t(i,j) = [-d_t^2 m_t(i,j) - \lambda_t - \mu_t(j) + \mu_{t-1}(i)]_+ \\
  \frac{\partial D}{\partial \lambda_t} &= d_t^{-2} \sum_{i,j} (\lambda_t + \xi_t(i,j) + \mu_t(j) - \mu_{t-1}(i)) + \sum_{i,j} m_t(i,j) - 1 \qquad \qquad \forall t \in [L-1]\\
  \frac{\partial D}{\partial \mu_t(i)} &= d_t^{-2} \sum_j (\lambda_t + \xi_t(j,i) + \mu_t(i) - \mu_{t-1}(j)) \qquad \qquad \qquad \qquad \qquad \qquad \forall t \in [L-2] \\
  &\phantom{=}+ d_{t+1}^{-2} \sum_j (\lambda_{t+1} + \xi_{t+1}(i,j) + \mu_{t+1}(j) - \mu_t(i)) + \sum_j m_t(j,i) - \sum_j m_{t+1}(i,j),
\end{align*}
where we further define $\mu_{L-1}(j) := 0$.  Obviously it takes $O(Lm^2)$ to compute all the gradients, and so is $\cbr{\xi_t(i,j)}$.


\end{document}